\documentclass{article}

\usepackage[main, final]{neurips_2026}

\usepackage[utf8]{inputenc} 
\usepackage[T1]{fontenc}    
\usepackage[colorlinks,allcolors=blue]{hyperref}      
\usepackage{url}            
\usepackage{booktabs}       
\usepackage{amsfonts}       
\usepackage{nicefrac}       
\usepackage{microtype}      
\usepackage{xcolor}         
\usepackage{amsmath}
\usepackage{amssymb}
\usepackage{mathtools}
\usepackage{amsthm}
\usepackage{bm,bbm,dsfont}
\usepackage{enumitem}
\usepackage{subcaption,wrapfig}
\usepackage{algorithm,algorithmic}
\usepackage{booktabs,tabularx,multirow,threeparttable}

\theoremstyle{plain}
\newtheorem{theorem}{Theorem}[section]
\newtheorem{proposition}[theorem]{Proposition}

\theoremstyle{definition}

\newtheorem{assumption}[theorem]{Assumption}
\theoremstyle{remark}
\newtheorem{remark}[theorem]{Remark}

\renewcommand{\tilde}{\widetilde} 
\renewcommand{\hat}{\widehat}

\DeclareMathOperator*{\argmin}{arg\,min}
\newcommand{\norm}[1]{\left\lVert#1\right\rVert}

\newtheorem{innercustomthm}{Theorem}
\newenvironment{customthm}[1]
{\renewcommand\theinnercustomthm{#1}\innercustomthm}
{\endinnercustomthm}

\usepackage{array}
\newcolumntype{C}{>{\centering\arraybackslash}X}

\allowdisplaybreaks

\title{Transfer Learning Through Conditional Quantile Matching}

\author{%
  Yikun Zhang\thanks{Work done when the first author was an intern at Meta and affiliated with the Department of Statistics, University of Washington. Correspondence to: \texttt{yikunz@uchicago.edu}} \\
  Department of Statistics, University of Chicago\\
  NSF-Simons AI Institute for the Sky (SkAI Institute)\\
  Chicago, IL \\
  \texttt{yikunz@uchicago.edu} \\
  \And
  Steven Wilkins-Reeves \quad Wesley Lee \quad Aude Hofleitner \\
  Central Applied Science, Meta \\
  Menlo Park, CA \\
 \texttt{\{stevewr, wtlee, aude\}@meta.com} \\
}

\begin{document}

\maketitle

\begin{abstract}
  We introduce a transfer learning framework for regression that leverages heterogeneous source domains to improve predictive performance in a data-scarce target domain. Our approach fits a conditional generative model to each source domain and calibrates a shared linear combination of the generated responses via conditional quantile matching at the observed target covariates. 
  The calibration aligns the conditional response distributions while accommodating simultaneous shifts in both the covariate and conditional response distributions. The calibrated responses then serve as pseudo-labels to augment the target data for a broad class of downstream learning tasks.
  Theoretically, we derive an excess risk bound for an empirical risk minimizer (ERM) trained on the augmented data, separating the benefit of an enlarged training sample from errors due to density ratio estimation, source generator estimation, conditional calibration, and irreducible transfer bias.
  Empirically, simulations and real-data applications show that the
  proposed method can improve target-domain prediction when the transfer bias is small.
\end{abstract}

\section{Introduction}
\label{sec:intro}

The rapid growth in the volume and heterogeneity of data has created new opportunities for machine learning (ML) methods to improve prediction in domains where labeled observations are scarce. Transfer learning seeks to exploit these opportunities by leveraging information from one or more related but heterogeneous source domains to improve learning in a target domain \citep{torrey2010transfer,weiss2016survey}. Commonly, we observe data from $K$ source domains, $\left\{\left(X_i^{(k)}, Y_i^{(k)} \right)\right\}_{i=1}^{n_k} \sim P^{(k)}$ for $k=1,...,K$, together with limited data from a target domain $\left\{\left(X_i^{(0)}, Y_i^{(0)} \right)\right\}_{i=1}^{n_0} \sim P^{(0)}$, where $P^{(j)},j=0,1,...,K$ are probability distributions supported on a common product space $\mathcal{X}\times \mathcal{Y}$. The goal is to use the source data to enhance prediction in the target domain, especially when $n_0$ is small.

The success of transfer learning depends critically on how well the source and target distributions align. Large discrepancies can lead to negative transfer, whereby incorporating source data degrades performance in the target domain \citep{wang2019characterizing,zhang2022survey}. To control this risk, much of the existing literature imposes structural assumptions on how source and target distributions may differ. Two widely studied assumptions are: 
\vspace{-1mm}
\begin{enumerate}[label=(\roman*),leftmargin=*]
	\setlength\itemsep{0.1em}
	\item \emph{Covariate Shift} \citep{shimodaira2000improving}: $P^{(k)}(y|x) = P^{(0)}(y|x)$ but $P^{(k)}(x)\neq P^{(0)}(x)$;
	\item \emph{Label Shift} \citep{saerens2002adjusting,nguyen2016continuous}: $P^{(k)}(x|y) = P^{(0)}(x|y)$ but $P^{(k)}(y)\neq P^{(0)}(y)$.
\end{enumerate}


Recent work has sought to relax these assumptions by introducing invariant or transformed features \citep{pan2010domain,gong2016domain}, latent variables \citep{tsai2024proxy}, or localized shift models \citep{wilkins2024multiply}. Nevertheless, empirical evidence suggests that both covariate and label shift assumptions are often violated in practice \citep{zhang2015multi,schrouff2022diagnosing}, or may hold only for a subset of source domains. Moreover, existing methods that attempt to identify transferable sources typically require additional tuning parameters \citep{bai2024transfer} or sample splitting of the already limited target data \citep{tian2023transfer,wang2023transfer}, which can further limit their effectiveness.

\begin{figure}
	\centering
	\includegraphics[width=0.8\linewidth]{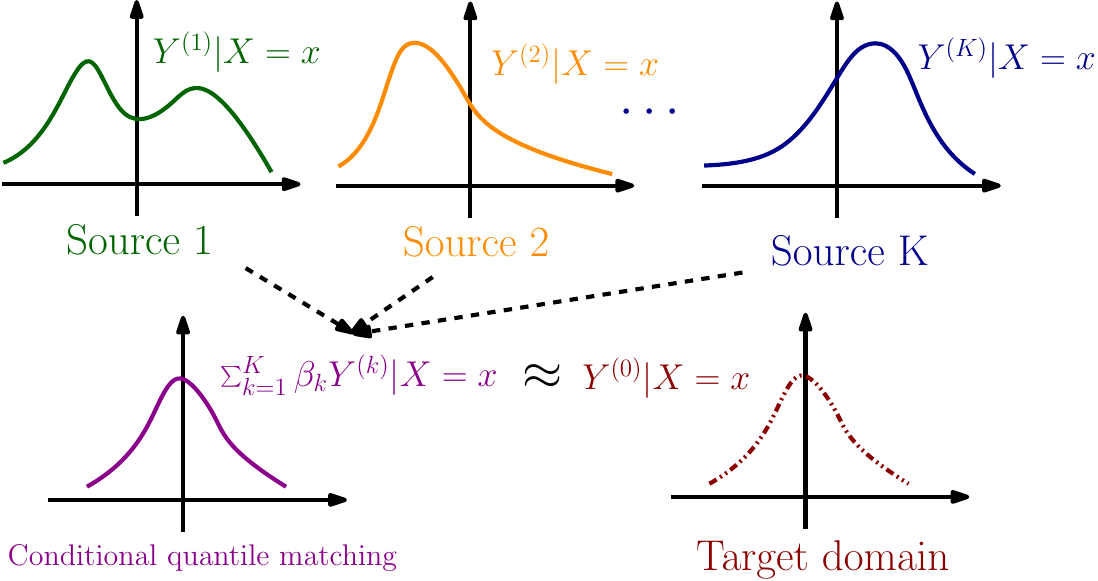}
	\caption{Illustration of conditional quantile matching at a fixed covariate value, where the intercept $\beta_0$ is omitted from the schematic.}
	\label{fig:condQM_illust}
	\vspace{-3mm}
\end{figure}

In this paper, we propose a novel transfer learning framework for settings in which both the covariate and conditional response
distributions may vary across domains. We first learn a conditional generative model in each source domain and calibrate a shared linear combination of their generated responses by minimizing a proper scoring rule at the observed target covariates. This approach performs synthetic data augmentation for the target domain under the general distributional shift setting:
\begin{equation}
	\label{target_shift}
	P^{(i)}(x)\neq P^{(j)}(x) \quad \text{ and } \quad P^{(i)}(y|x)\neq P^{(j)}(y|x)
\end{equation}
for $i\neq j$ and $i,j=0,...,K$. \autoref{fig:condQM_illust} illustrates the central idea. At a given covariate value $x$, the fitted source generators produce conditional response draws, which are combined using a common coefficient vector.
The key advantage of our transfer learning framework lies in its design to calibrate the conditional response distribution, rather than transferring model parameters or feature representations as commonly proposed in prior works reviewed in \autoref{subsec:related_work}. Moreover, the shared linear calibration implicitly regularizes the contribution of each source domain, automatically attenuating sources whose conditional responses are poorly aligned with the target. 
However, the calibration does not guarantee successful transfer. When the target conditional distribution is poorly approximated by the calibrated conditional family, the resulting transfer bias remains in our risk bound.


The contributions of this paper are summarized as follows.
\vspace{-3mm}
\begin{enumerate}[leftmargin=*]
	\setlength\itemsep{0.1em}
	\item {\bf Methodology:} We introduce a general transfer learning framework for regression based on conditional quantile matching (TLCQM) in \autoref{sec:method}. The calibrated responses provide pseudo-labels for data augmentation and can be used with a broad range of downstream regression methods.
		
	\item {\bf Theory:} We derive an excess-risk bound for ERMs trained on the augmented data in \autoref{sec:theory}, separating standard generalization error from density ratio estimation error and conditional transfer error. We further establish a finite-sample bound for the conditional calibration estimator that quantifies the effects of source-generator error, optimization error, and irreducible transfer bias.
		
	\item {\bf Experiments:} Through experiments on simulated and real datasets, we demonstrate that the proposed framework can improve prediction accuracy over target-only learning and other competing transfer learning methods when transfer bias is small; see \autoref{sec:experiments}. The code used for our experiments is available at \url{https://github.com/facebookresearch/TLCQM}.
\end{enumerate}

\subsection{Notation}

Throughout this paper, we consider a continuous response variable with $\mathcal{Y}\subset \mathbb{R}$ and allow a general (measurable) covariate space $\mathcal{X}$, though our proposed framework can be readily extended to settings involving multivariate and/or discrete response variables; see \autoref{sec:discuss} for details. We say a function $f:\mathcal{X}\to \mathbb{R}$ is in $L_2(P)$ if $\norm{f}_{L_2(P)} := \left[\int_{\mathcal{X}} |f(x)|^2 dP(x)\right]^{\frac{1}{2}} < \infty$ for a probability measure $P$. We adopt standard asymptotic notation: for deterministic sequences ${h_n}$ and ${g_n}$ with $g_n > 0$, we write $h_n = O(g_n)$ if $\frac{|h_n|}{g_n} \leq C$ for some absolute constant $C > 0$ and all sufficiently large $n$, and $h_n = o(g_n)$ if $\frac{|h_n|}{g_n} \to 0$ as $n \to \infty$. For a random sequence $X_n$ and a positive deterministic sequence ${h_n}$, $X_n=o_P(h_n)$ means that $\frac{X_n}{h_n}$ converges to 0 in probability, while $X_n=O_P(h_n)$ indicates that $\frac{X_n}{h_n}$ is bounded in probability as $n\to \infty$.
We further write $a_n \lesssim b_n$ (or equivalently $b_n \gtrsim a_n$) if $a_n \leq C b_n$ for some constant $C>0$ and all sufficiently large $n$. When both $a_n \lesssim b_n$ and $a_n \gtrsim b_n$ hold, we denote the asymptotic equivalence by $a_n \asymp b_n$.

\section{Proposed Framework}
\label{sec:method}


Algorithm~\ref{algo:CQM} summarizes the proposed TLCQM framework. The procedure fits a conditional generative model in each source domain, calibrates a shared linear combination of the generated responses at the observed target covariates, and uses the calibrated synthetic responses to augment the target sample.
Steps 1 and 3 are central to the methodology and are detailed in \autoref{subsec:engress} and \autoref{subsec:quant_match}, respectively.

\begin{algorithm}[t]
	\caption{Data Augmentation with Conditional Quantile Matching (TLCQM)}
	\label{algo:CQM}
	\begin{algorithmic}
		\STATE {\bf Input:} Target data $\mathcal{D}_T=\left\{\left(X_i^{(0)}, Y_i^{(0)} \right)\right\}_{i=1}^{n_0}$ and source data $\mathcal{D}_S^{(k)}=\left\{\left(X_i^{(k)}, Y_i^{(k)} \right)\right\}_{i=1}^{n_k},k=1,...,K$.
		\STATE {\bf Step 1:} For each source domain $k$, fit a conditional generative model $\hat{P}^{(k)}(y|x)$ using the source data $\mathcal{D}_S^{(k)}$.
		
		\STATE {\bf Step 2:} For each target covariate vector $X_i^{(0)}$, generate $M$ synthetic responses independently from each fitted generator and obtain $\left\{\left(X_i^{(0)}, \hat{\bm{Y}}_{ij} \right): i=1,...,n_0, j=1,...,M \right\}$, where  $\hat{\bm{Y}}_{ij}=\left(\hat{Y}_{ij}^{(1)},..., \hat{Y}_{ij}^{(K)}\right) \, \text{ with } \, \hat{Y}_{ij}^{(k)}\sim \hat{P}^{(k)}(y|X_i^{(0)})$.
		
		\STATE {\bf Step 3:} Compute the conditional quantile calibration estimator $\hat{\bm{\beta}}$ by solving \eqref{quantile_ls}.
		
		\STATE {\bf Step 4:} Augment the target domain with the synthetic data $\mathcal{D}_A=\bigcup_{k=1}^K\left\{\left(X_i^{(k)}, \hat{\bm{\beta}}^T \hat{\bm{V}}_i^{(k)} \right) \right\}_{i=1}^{n_k}$,
		with $\hat{\bm{V}}_i^{(k)} = \left(1,\hat{Y}_i^{(1,k)},...,\hat{Y}_i^{(K,k)} \right)^T\in \mathbb{R}^{K+1}$ and $\hat{Y}_i^{(j,k)} \sim \hat{P}^{(j)}(y|X_i^{(k)})$ or any predicted value for $Y_i^{(j)}$ from the $j$-th source generator evaluated at the covariate $X_i^{(k)}$.
	
		\STATE {\bf Step 5 (Optional):} Estimate $w_k(x)=\frac{dP_X^{(0)}(x)}{dP_X^{(k)}(x)}$ for $k=1,...,K$ to correct for covariate shifts.
		\STATE {\bf Output:} The final augmented data $\mathcal{D}_F = \mathcal{D}_T\cup \mathcal{D}_A$.
	\end{algorithmic}
\end{algorithm}

\subsection{Learning Conditional Generative Models via Engression}
\label{subsec:engress}

Depending on data availability and problem structure, a variety of conditional generative models can be employed in Step 1 of Algorithm~\ref{algo:CQM}. Here, we consider a neural network-based distributional regression method known as ``engression'' \citep{shen2025engression} as a default choice for its simplicity, flexibility, and empirical effectiveness. Specifically, for each source domain $k$, we model the conditional distribution $P^{(k)}(y|x)$ through a measurable function $g^{(k)}: \mathcal{X} \times \mathcal{E} \to \mathcal{Y}$, which maps covariates $x$ and pre-specified noise vectors $\eta \sim P_{\eta}$ (\emph{e.g.}, Gaussian or uniform) to responses. The induced pushforward measure satisfies $g^{(k)}(x,\cdot)_\# P_{\eta} = P^{(k)}(\cdot|x)$.
Under the engression framework, the population-level estimator is obtained by solving
\begin{align}
	\label{popul_engress}
	\begin{split}
		g^{(k)} \in \argmin_{g\in \mathcal{G}} &\, \mathbb{E}_{ (X,Y,\eta,\eta')\sim P^{(k)}\times P_{\eta}\times P_{\eta}}\Big[|Y-g(X,\eta)| - \frac{1}{2}\left|g(X,\eta) - g(X,\eta')\right|\Big]
	\end{split}
\end{align}
over a generator class $\mathcal{G}$ that contains a representation of the true conditional law, where $\eta,\eta'$ are independent draws from $P_{\eta}$. 
The first term in \eqref{popul_engress} corresponds to an $L_1$-risk rewarding agreement with observed responses. The second negative term in \eqref{popul_engress} promotes variability in the generated responses across noise realizations, preventing collapse to a deterministic predictor and enabling $g^{(k)}$ to capture the conditional distribution $P^{(k)}(y|x)$. 
In practice, we draw $m\geq 2$ independent and identically distributed (i.i.d.) samples $\eta_{ij},j=1,...,m$ for each observation $\left(X_i^{(k)},Y_i^{(k)}\right)$ and solve the empirical optimization problem
\begin{align*}
	\hat{g}^{(k)} \in \argmin_{g\in \mathcal{G}} \frac{1}{n_k} \sum_{i=1}^{n_k} &
	\left[\frac{1}{m} \sum_{j=1}^m
	\left|Y_i^{(k)} - g(X_i^{(k)}, \eta_{ij})\right| \right.\\
	&\quad \left. -\frac{1}{2m(m-1)} \sum_{j=1}^m \sum_{j'=1}^m
	\left|g(X_i^{(k)},\eta_{ij}) - g(X_i^{(k)},\eta_{ij'})\right| \right].
\end{align*}
Following Step 2 of Algorithm~\ref{algo:CQM}, synthetic responses are then generated at target covariates via $\hat{Y}_{ij}^{(k)} = \hat{g}^{(k)}(X_i^{(0)},\eta_{ij}^{(k)})$, where $\left\{\eta_{ij}^{(k)}\right\}_{j=1}^M$ are i.i.d. noise vectors from the same pre-specified distribution and $M\geq2$ controls the Monte Carlo accuracy of the calibration objective.

\subsection{Conditional Quantile Matching}
\label{subsec:quant_match}

For $x\in \mathcal{X}$, let $\bm{V}(x,\bm{\eta})=\left(1,g^{(1)}(x,\eta_1),...,g^{(K)}(x,\eta_K)\right)^T\in \mathbb{R}^{K+1}$, where $\eta_1,...,\eta_K \sim P_{\eta}$ are mutually independent. We also denote the conditional cumulative distribution function (CDF) of $\bm{\beta}^T\bm{V}(X,\bm{\eta})$ given $X=x$ by $G_{\bm\beta}(\cdot|x)$ and the target conditional CDF by $F^{(0)}(\cdot|x)$. Here, we adopt the linear form $\bm{\beta}^T\bm V(x,\bm{\eta})$ for two major reasons. First, it is computationally more efficient than general optimal transport-based distributional alignment methods \citep{chernozhukov2017monge}. Second, the linear structure in \eqref{quantile_ls_true} implicitly regularizes the contribution of each source domain and mitigates negative transfer.

At the population level, our conditional quantile calibration parameter $\bm{\beta}_* \in \mathbb{R}^{K+1}$ satisfies
\begin{equation}
	\label{quantile_ls_true}
	\bm{\beta}_*\in \argmin_{\bm{\beta} \in \Theta} \mathcal{C}(\bm{\beta}):=  \mathbb{E}\left[\mathrm{CRPS}\left(G_{\bm\beta}(\cdot|X^{(0)}),Y^{(0)}\right)\right].
\end{equation}
Here, we optimize over a compact set $\Theta=\{\bm{\beta}\in \mathbb{R}^{K+1}:\norm{\bm{\beta}}_1 \leq B_\beta\}$, which can be enforced by adding a $L_1$ regularization term to the objective \eqref{quantile_ls_true} in practice. Additionally,
\begin{equation}
\label{CRPS}
\mathrm{CRPS}\left(G,y\right) = \int_{\mathbb{R}} \left[G(t) -\mathds{1}\left\{y\leq t\right\}\right]^2 dt =2\int_0^1 \rho_\alpha\left(y-Q_G(\alpha)\right) d\alpha
\end{equation}
is the continuous ranked probability score (CRPS; \citealt{matheson1976scoring,hersbach2000decomposition}), where $\rho_\alpha(u)=u\left[\alpha-\mathds{1}\{u<0\}\right]$ is the quantile loss and $Q_G(\alpha)$ is the $\alpha$-quantile of the CDF $G$. The second expression of CRPS in \eqref{CRPS} illustrates that \eqref{quantile_ls_true} minimizes the \emph{conditional quantile loss} of $\bm{\beta}^T\bm{V}(x,\bm{\eta})$ with $Y^{(0)}$ given $X=x$ integrated over $\alpha\in [0,1]$ and averaged over target covariates.

When $G$ has a finite first moment, CRPS has the following more computationally tractable representation as:
\begin{equation}
	\label{CRPS2}
	\mathrm{CRPS}\left(G,y\right) = \mathbb{E}_{Z \sim G}|y-Z| - \frac{1}{2}\mathbb{E}_{Z,Z'\stackrel{\rm ind}{\sim}G}|Z-Z'|,
\end{equation}
where $Z$ and $Z'$ are independent random variables with CDF as $G$; see Lemma 2.2 in \cite{baringhaus2004new}. As proved by \cite{gneiting2007strictly}, CRPS is a strictly proper score. In particular, the excess of $\mathcal{C}(\bm{\beta})$ in \eqref{quantile_ls_true} over the oracle minimum $\mathcal{C}_0 = \mathbb{E}\left\{\mathrm{CRPS}\left[F^{(0)}(\cdot|X^{(0)}),Y^{(0)}\right]\right\}$ is given by
\begin{equation}
	\label{crps_conditional_identity}
	\mathcal C(\bm\beta)-\mathcal C_0=\mathbb E_{X^{(0)}}\left\{\int_{\mathbb R}\left[F^{(0)}(t|X^{(0)})-G_{\bm\beta}(t|X^{(0)})\right]^2 dt\right\} =:\Delta(\bm\beta).
\end{equation}

According to \eqref{CRPS2}, we compute the sample conditional quantile calibration estimator in Step 3 of Algorithm~\ref{algo:CQM} as:
\begin{equation}
	\label{quantile_ls}
	\begin{aligned}
		\hat{\bm\beta} &\in\argmin_{\bm\beta\in\Theta} \hat{\mathcal C}_{n_0,M}(\bm{\beta}) \qquad \text{ with } \\
		\hat{\mathcal C}_{n_0,M}(\bm{\beta}) &:= \frac{1}{n_0}\sum_{i=1}^{n_0}\left\{\frac{1}{M}\sum_{j=1}^M \left|Y_i^{(0)}- \bm{\beta}^T\hat{\bm V}_{ij}\right| -\frac{1}{2M(M-1)}\sum_{j, j'=1}^M \left|\bm{\beta}^T\hat{\bm V}_{ij} -\bm{\beta}^T\hat{\bm V}_{ij'}\right|\right\}
	\end{aligned}
\end{equation}
where $\hat{\bm V}_{ij}=\left(1,\hat Y_{ij}^{(1)},...,\hat Y_{ij}^{(K)}\right)^T\in \mathbb{R}^{K+1}$ consists of generated synthetic responses $\hat{Y}_{ij}^{(k)} = \hat{g}^{(k)}(X_i^{(0)},\eta_{ij}^{(k)}), j=1,...,M$ at each $X_i^{(0)}$ for $k=1,...,K$. Conditional on source generators $\hat{g}^{(k)}, k=1,...,K$, fitted independently of the target observations, the objective $\hat{\mathcal{C}}_{n_0,M}(\bm{\beta})$ in \eqref{quantile_ls} is an unbiased estimator (\emph{i.e.}, U-statistic) of $\mathbb{E}\left[\mathrm{CRPS}\left(\hat{G}_{\bm{\beta}}(\cdot|X^{(0)}), Y^{(0)}\right)\right]$ for any fixed $\bm{\beta}\in \Theta$. When $\hat{g}^{(k)}$ is consistent for $g^{(k)}$, $\hat{\mathcal{C}}_{n_0,M}(\bm{\beta})$ also consistently estimates $\mathcal{C}(\bm{\beta})$. 


While the objective in \eqref{quantile_ls} is non-smooth and generally non-convex, it is a difference of a proper closed convex function and a continuous convex function, which can be optimized by the classical difference-of-convex algorithm (DCA; \citealt{tao1997convex,tao1998dc}). Here, we use a proximal DCA update related to the framework of \cite{wen2018proximal}. Writing $\hat{\mathcal{C}}_{n_0,M}(\bm{\beta})=\Phi(\bm\beta)-\Psi(\bm\beta)$, where
$$\Phi(\bm{\beta}) = \frac{1}{n_0}\sum_{i=1}^{n_0}\frac{1}{M}\sum_{j=1}^M \left|Y_i^{(0)}- \bm{\beta}^T\hat{\bm V}_{ij}\right|, \quad \Psi(\bm{\beta}) = \frac{1}{n_0}\sum_{i=1}^{n_0}\frac{1}{2M(M-1)}\sum_{j, j'=1}^M \left|\bm{\beta}^T\hat{\bm V}_{ij} - \bm{\beta}^T\hat{\bm V}_{ij'}\right|$$
are both convex, the proximal DCA chooses a subgradient $\bm{s}^{(r)} \in \partial \Psi(\bm\beta^{(r)})$ and updates
\begin{equation}
\label{dca_update}
\bm{\beta}^{(r+1)}\in\argmin_{\bm{\beta}\in \Theta} \left\{\Phi(\bm\beta)- \bm{s}^{(r)T}\bm\beta + \frac{\gamma}{2}\|\bm\beta-\bm\beta^{(r)}\|_2^2\right\},\qquad \text{ for some }\gamma>0.
\end{equation}
The convergence results of this proximal DCA update \eqref{dca_update} have been studied in \cite{wen2018proximal}.

For squared-error prediction, we recommend using the calibrated conditional mean rather than an individual generated response as the augmented label in Step 4 of Algorithm~\ref{algo:CQM}. In particular, at any source covariate $X_i^{(k)}$, we replace $\hat{\bm{\beta}}^T \hat{\bm{V}}_i^{(k)}$ with 
\begin{equation}
\label{cond_mean_pred}
\hat{\mu}_{\hat{\bm\beta}}(X_i^{(k)})=\hat{\beta}_0 + \sum_{j=1}^K \hat{\beta}_j\left[\frac{1}{M'}\sum_{\ell=1}^{M'}\hat g^{(j)}(X_i^{(k)},\eta_{\ell}^{(j)})\right],
\end{equation}
where $M'>0$ is an integer controlling the Monte Carlo accuracy of the estimated conditional mean. This construction removes unnecessary response noise from the pseudo-labels and leads to a smaller conditional transfer error; see \autoref{thm:transfer_learn_excess_risk} for details.

\begin{remark}
\label{remark:cond_mean}
Let $\mu^{(0)}(x)=\mathbb{E}(Y^{(0)}|X^{(0)}=x)$ be the true regression function in the target domain and $\mu^{(k)}(x)=\mathbb{E}_{\eta}\left[ g^{(k)}(x,\eta)\right]$. When $g^{(k)}$ is correctly specified (see Assumption~\ref{assump:quantile_match} below), $\mu^{(k)}(x)=\mathbb{E}(Y^{(k)}|X^{(k)}=x)$. The conditional mean induced by $\bm{\beta}$ is written as $\mu_{\bm{\beta}}(x)=\beta_0+\sum_{k=1}^K\beta_k \cdot \mu^{(k)}(x)$. If $\mu^{(0)}$ lies outside the calibrated mean class $\left\{\mu_{\bm{\beta}}:\bm{\beta}\in \Theta \right\}$, no common coefficient vector can eliminate the conditional mean discrepancy. This irreducible discrepancy contributes to the transfer bias in our theoretical analysis; see \autoref{thm:conditional_calibration_rate} for details.
\end{remark}

\section{Theoretical Analysis}
\label{sec:theory}

This section theoretically compares the prediction risks of ERMs trained solely on the target domain with those trained on the augmented data $\mathcal{D}_F$ produced by our TLCQM framework (Algorithm~\ref{algo:CQM}).

Let $\ell:\mathcal{Y}\times \mathcal{Y} \to \mathbb{R}$ be a symmetric loss function. The population risk of a prediction function $f\in \mathcal{F}$ on the target domain is defined as $R(f):=\mathbb{E}\left[\ell\left(Y^{(0)},f(X^{(0)})\right)\right]$, and let $f^{(0)}\in\argmin_{f\in \mathcal{F}} R(f)$ be the population risk minimizer over the function class $\mathcal{F}$. Then, the excess risk of any candidate function $f:\mathcal{X}\to \mathcal{Y}$ is defined as $R(f) - R(f^{(0)}) = \mathbb{E}\left[\ell\left(Y^{(0)}, f(X^{(0)})\right) \right] - \mathbb{E}\left[\ell\left(Y^{(0)}, f^{(0)}(X^{(0)})\right) \right]$.

We analyze ERMs corresponding to the following two different training strategies.
\vspace{-1mm}
\begin{itemize}[leftmargin=*]
	\setlength\itemsep{-0.5em}
	\item {\bf (Target-only ERM)} We define $\hat{f}^{(0)}:\mathcal{X}\to \mathcal{Y}$ as:
	\begin{align}
		\label{target_only_pred}
		\begin{split}
			\hat{f}^{(0)} &\in \argmin_{f\in \mathcal{F}} \frac{1}{n_0} \sum_{i=1}^{n_0} \ell\left(Y_i^{(0)}, f(X_i^{(0)}) \right).
		\end{split}
	\end{align}
	
	\item {\bf (TLCQM ERM)} Using the augmented data $\mathcal{D}_F$ from Algorithm~\ref{algo:CQM}, the transfer learning prediction function is defined as:
	\begin{align}
		\label{tranfer_learn_pred}
		\begin{split}
			\hat{f}^{(0,tl)} \in \argmin_{f\in \mathcal{F}} \frac{1}{N} \left\{\sum_{i=1}^{n_0} \ell\left(Y_i^{(0)}, f(X_i^{(0)})\right) + \sum_{k=1}^K \sum_{i=1}^{n_k} \hat{w}_k(X_i^{(k)}) \cdot \ell\left(\hat{Z}_i^{(k)}, f(X_i^{(k)})\right)\right\},
		\end{split}
	\end{align}
	where $N=\sum_{k=0}^K n_k$ is the combined sample size of both source and target domains, $\hat{w}_k(x)$ is the estimated density ratio $w_k(x)=\frac{dP_X^{(0)}(x)}{dP_X^{(k)}(x)}$ for $k=1,...,K$, and $\hat{Z}_i^{(k)}$ is either the Monte Carlo augmented conditional mean in \eqref{cond_mean_pred} or an independent calibrated draw $\hat{\bm{\beta}}^T \hat{\bm{V}}_i^{(k)}$.
\end{itemize}

The density ratio correction in \eqref{tranfer_learn_pred} is not necessary when the function class $\mathcal{F}$ for the prediction task is correctly specified \citep{sugiyama2007covariate}. In practice, however, incorporating importance weights improves robustness to model misspecification and often stabilizes numerical optimization.

\subsection{Assumptions}

To state a finite-sample result without imposing uniform entropy conditions on data-dependent nuisance estimators, we assume that the fitted generators $\hat{g}^{(k)},k=1,...,K$, calibration coefficients $\hat{\bm{\beta}}$, and density ratio estimators $\hat{w}_k,k=1,....,K$ are learned from auxiliary data independent of the observations used to train the downstream ERM. The same guarantee can be implemented without losing data efficiency through cross-fitting \citep{schick1986asymptotically,newey2018cross}.

\begin{assumption}[Basic regularity conditions]
	\label{assump:basic_reg}\noindent\vspace{-2mm}
	\begin{enumerate}[label=(\alph*),leftmargin=*]
		\item On a compact interval containing $\mathcal{Y}$ and all calibrated responses used for augmentation, there exists a constant $B_{\ell}>0$ such that $|\ell(y,y')|\leq B_{\ell}$ and $|\ell(y,y') - \ell(y,y'')|\leq B_{\ell}|y' -y''|$.
		
		\item $P_X^{(0)}$ is absolutely continuous with respect to each $P_X^{(k)}$, and there exists a constant $B_w>0$ such that $0\leq w_k(x)=\frac{dP_X^{(0)}(x)}{dP_X^{(k)}(x)} \leq B_w$ almost surely for all $k=1,...,K$.
	\end{enumerate}
\end{assumption}

\begin{assumption}[Boundedness conditions for excess risk]
	\label{assump:quantile_match}
	The conditional distribution $P^{(k)}(Y^{(k)}|X=x)$ is represented by $g^{(k)}(x,\eta)$ for $k=0,1,...,K$, where $\eta$ is an independent random vector with a pre-specified distribution. For a constant $B_g>0$, suppose $\sup_{(x,\eta)}|g^{(k)}(x,\eta)|\leq B_g$ for $k=0,1,...,K$ and $\sup_{(x,\eta)}|\hat g^{(k)}(x,\eta)|\leq B_g$ for $k=1,...,K$.
\end{assumption}

Assumption~\ref{assump:basic_reg}(a) is a standard boundedness and Lipschitz continuity condition on both arguments of the loss function, which holds, for example, when $\mathcal{Y}$ is compact or when the loss function $\ell$ is truncated at a certain value $B_{\ell}>0$. Assumption~\ref{assump:basic_reg}(b) prevents instability from extreme density ratios (or importance weights) under covariate shifts \citep{gretton2009covariate,reddi2015doubly}, so that these weights can be consistently estimated. Assumption~\ref{assump:quantile_match}, together with the compactness of $\Theta$, uniformly bounds the true and fitted conditional generative models, which also resemble the compactness assumption on $\mathcal{Y}$.

\subsection{Target-Only Excess Risk}

For a function class $\mathcal{F}=\left\{f: \mathcal{X}\to \mathcal{Y}\right\}$ and independent covariates $X_1,...,X_n$ (not necessarily identically distributed), let
$\hat{\mathrm{Rad}}_n(\mathcal{F})
=\frac{1}{n} \mathbb{E}_\sigma[\sup_{f\in \mathcal{F}}\left|\sum_{i=1}^n \sigma_i \cdot f(X_i)\right|]$,
where $\sigma_i$'s are i.i.d. Rademacher random variables, and denote
$\mathrm{Rad}_n(\mathcal{F}) = \mathbb{E}\left[\hat{\mathrm{Rad}}_n(\mathcal{F})\right]$. For $n=N$, the expectation is taken over the joint distribution of the pooled target and source covariates.

\begin{proposition}
	\label{prop:target_only_excess}
	Under Assumption~\ref{assump:basic_reg}(a), the excess risk of the target-only prediction function $\hat{f}^{(0)}$ in \eqref{target_only_pred} satisfies, with probability at least $1-\delta$,
	$$R(\hat{f}^{(0)}) - R(f^{(0)}) \lesssim \mathrm{Rad}_{n_0}(\mathcal{F}) + \sqrt{\frac{\log(2/\delta)}{n_0}}.$$
\end{proposition}

Proposition~\ref{prop:target_only_excess} is a classical result in statistical learning theory; see, \emph{e.g.}, Theorem 26.5 in \citet{shalev2014understanding}. The upper bound remains valid if $\mathrm{Rad}_{n_0}(\mathcal{F})$ is replaced by its empirical counterpart $\hat{\mathrm{Rad}}_{n_0}(\mathcal{F})$. In many common settings, $\mathrm{Rad}_{n_0}(\mathcal{F})$ scales as $O\left(\frac{1}{\sqrt{n_0}}\right)$ up to some $\log n_0$ factor, with constants depending critically on the complexity of the functional class $\mathcal{F}$. The representative bounds are summarized in \autoref{tab:rad_complexity}.

\subsection{Excess Risk Under Our TLCQM framework}

For $\bm{\beta}\in \Theta$, let $\hat{G}_{\hat{\bm{\beta}}}(\cdot|x)$ be the conditional CDF of $\hat{\bm{\beta}}^T \hat{\bm{V}}(x,\bm{\eta}) = \hat{\beta}_0 + \sum_{k=1}^K \hat{\beta}_k\cdot \hat{g}^{(k)}(x,\eta_k)$, and we define its conditional CRPS error by
\begin{equation}
\label{estimated_conditional_crps}
\hat{\Delta}(\hat{\bm{\beta}}) :=\mathbb{E}_{X^{(0)}}\left\{\int_{\mathbb{R}}\left[F^{(0)}(t|X^{(0)})-\hat{G}_{\hat{\bm{\beta}}}(t|X^{(0)})\right]^2 dt \right\}.
\end{equation}

\begin{theorem}
	\label{thm:transfer_learn_excess_risk}
Suppose that Assumptions~\ref{assump:basic_reg} and \ref{assump:quantile_match} hold and that $\left(\hat{\bm{g}},\hat{\bm{\beta}},\hat{\bm{w}}\right)$ is fitted on auxiliary data independent of the observations used in \eqref{tranfer_learn_pred}. Assume further that $\ell(y,y') = (y-y')^2$ and $\mathcal{F}$ is a closed convex subset of $L_2(P_X^{(0)})$. Then, the transfer learning prediction function $\hat{f}^{(0,tl)}$ in \eqref{tranfer_learn_pred} under augmented responses as \eqref{cond_mean_pred} has its excess risk satisfying that, with probability at least $1-\delta$,
\begin{align*}
	R(\hat{f}^{(0,tl)}) - R(f^{(0)}) &\lesssim \underbrace{\mathrm{Rad}_N(\mathcal{F}) + \sqrt{\frac{\log(1/\delta)}{N}}}_{\textbf{Standard generalization error}} + \underbrace{\frac{1}{N} \sum_{k=1}^K \norm{\hat{w}_k-w_k}_1}_{\textbf{Importance weight error}} \\
	&\quad + \underbrace{\frac{(N-n_0)^2}{N^2}\left[\hat{\Delta}(\hat{\bm{\beta}}) + \frac{1}{M'}\right]}_{\textbf{Conditional transfer and Monte Carlo errors}}
\end{align*}
where $\norm{\hat{w}_k - w_k}_1 = \sum_{i=1}^{n_k}\left|\hat{w}_k(X_i^{(k)}) - w_k(X_i^{(k)}) \right|$ and $\delta \in (0,1)$. 

If $\mathcal{F}$ is not convex, $\ell$ is a general loss function satisfying Assumption~\ref{assump:basic_reg}(a) only, and the augmented labels are independent calibrated draws $\hat{\bm{\beta}}^T \hat{\bm{V}}_i^{(k)}$, then the last term becomes $\frac{N-n_0}{N} \sqrt{\hat{\Delta}(\hat{\bm{\beta}})}$.
\end{theorem}

The proof of \autoref{thm:transfer_learn_excess_risk} is in \autoref{app:proof_excess_risk}, where the transfer-learning excess risk is decomposed into three error terms. 
\vspace{-2mm}
\begin{itemize}[leftmargin=*]
	\setlength\itemsep{0.1em}
	\item \textbf{Standard generalization error:} This term mirrors the classical ERM bound but scales with the augmented sample size $N$ rather than the small target sample size $n_0$. Under standard complexity scaling, it is typically smaller when $N\gg n_0$.
	
	\item \textbf{Importance weight error:} This term captures the error induced by estimating density ratios under covariate shift. When the penalized risk minimization \citep{nguyen2007estimating} is applied, the rate is $O_P\left(\sum_{k=1}^K \left(n_k^{-\frac{1}{2+\tau}} + n_k^{-\frac{1}{4}}\right) \right)$ up to some $\log n_k$ factors (see Lemma 8 in \citealt{reddi2015doubly}). 
	This term vanishes faster than the target-only excess risk under mild growth of $n_k$ for $k=1,...,K$.
	
	\item \textbf{Conditional transfer error:} We further derive an upper bound for the conditional CRPS error \eqref{estimated_conditional_crps} in \autoref{thm:conditional_calibration_rate} below. In particular,
	\begin{equation}
	\label{cond_transfer_error}
		\hat{\Delta}(\hat{\bm{\beta}}) \leq \underbrace{\Delta(\bm{\beta}_*)}_{\textbf{Transfer bias}} + \underbrace{\sum_{k=1}^K \norm{\hat{g}^{(k)} - g^{(k)}}_{\infty}}_{\textbf{Distributional learning error}} + O_P\left(\sqrt{\frac{\log n_0 \left[\Delta(\bm{\beta}_*) + \sum_{k=1}^K \norm{\hat{g}^{(k)} - g^{(k)}}_{\infty} \right]}{n_0}}\right) 
	\end{equation}
    up to some optimization and Monte Carlo approximation errors. 

    The ``\textbf{transfer bias}'' term is the intrinsic approximation error of the population-level conditional quantile matching \eqref{quantile_ls_true} relative to the target distribution. By \eqref{crps_conditional_identity}, $\Delta(\bm{\beta}_*)=0$ if there exists $\bm{\beta}\in\Theta$ such that the conditional distributions of $\bm{\beta}^T V(X,\bm{\eta}) = \beta_0+ \sum_{k=1}^K \beta_k g^{(k)}(X^{(0)},\eta_k)$ equals $F^{(0)}(\cdot|X^{(0)})$ for $P_X^{(0)}$-almost every $x$, \emph{i.e.}, $P^{(0)}(Y|X)$ lies in the distribution family induced by the affine span of $P^{(k)}(Y|X),k=1,...,K$, a setting commonly assumed in multi-source transfer learning (see Section 3 in \citealt{turrisi2022multi}). In this case, the proposed framework achieves asymptotically unbiased transfer learning. 
    
    The ``\textbf{distributional learning error}'' depends on the expected uniform prediction errors of the source generators. Under standard smoothness assumptions, the optimal nonparametric error rate is $O\left(\sum_{k=1}^K n_k^{-\frac{s}{d+2s}}\right)$, where $s$ is the H\"older smoothness level of $g^{(k)}$ for $k=1,...,K$ \citep{stone1980optimal,stone1982optimal}. Related rates are attainable by sieve methods \citep{chen2013optimal} and neural networks \citep{imaizumi2023sup}. Within the engression framework, when $g^{(k)}$ is quadratic in $x$ and $\eta$, $\norm{\hat{g}^{(k)} - g^{(k)}}_{\infty}$ converges at rate $O_P\left(n_k^{-\frac{1}{3}}\right)$ up to a logarithmic factor.
	
	The third term in \eqref{cond_transfer_error} scales as $O_P\left(\sqrt{\frac{\log n_0}{n_0}}\right)$ times the square root of the sum of the transfer bias and distributional learning error. It is therefore of smaller order when both components vanish sufficiently quickly.
\end{itemize}

Finally, we emphasize that when the optimal quantile matching estimator \eqref{quantile_ls} is sparse, the $K$-dependent terms in \autoref{thm:transfer_learn_excess_risk} scale with the number of active source domains rather than the total number of source domains.

\subsection{Conditional Quantile Calibration Error}

\begin{theorem}
\label{thm:conditional_calibration_rate}
Let $\delta \in (0,1)$ and $\norm{\hat{g}^{(k)} - g^{(k)}}_{\infty} =\sup_{x\in \mathcal{X}} \mathbb{E}_{\eta}\left|\hat{g}^{(k)}(x,\eta) - g^{(k)}(x,\eta)\right|$ for $k=1,...,K$. Suppose that Assumption~\ref{assump:quantile_match} holds and that $\hat{\bm{g}}$ is independent of the target observations used in \eqref{quantile_ls}. If $\hat{\mathcal C}_{n_0,M}(\hat{\bm{\beta}}) \leq \inf_{\bm\beta\in\Theta} \hat{\mathcal{C}}_{n_0,M}(\bm{\beta}) + \epsilon_{\rm opt}$ in \eqref{quantile_ls}, then with probability at least $1-\delta$,
\begin{align*}
	\hat{\Delta}(\hat{\bm{\beta}}) - \Delta(\bm{\beta}_*) &\lesssim \sum_{k=1}^K \norm{\hat{g}^{(k)} - g^{(k)}}_{\infty} + \sqrt{\frac{\left(\log n_0 + \log(1/\delta)\right) \left[\Delta(\bm{\beta}_*) + \sum_{k=1}^K \norm{\hat{g}^{(k)} - g^{(k)}}_{\infty}\right]}{n_0}} \\
	&\quad + \frac{\log n_0 + \log(1/\delta)}{n_0} + \epsilon_{\rm opt}
\end{align*}
up to some Monte Carlo approximation errors of order $O\left(\frac{1}{\sqrt{M}}\right)$.
\end{theorem}

The proof of \autoref{thm:conditional_calibration_rate} is in \autoref{app:proof_cond_cal_err}.

\section{Numerical Experiments}
\label{sec:experiments}

In this section, we empirically evaluate the performance of several ML models under the proposed TLCQM framework in Algorithm~\ref{algo:CQM} and compare them with target-only ML models. We also benchmark TLCQM against several competing transfer learning approaches for regression under distributional shift, including (i) transfer learning with kernel ridge regression and automatic source selection (TKRR; \citealt{wang2023transfer}), (ii) deep transfer learning for conditional shift in regression (CDAR; \citealt{liu2021deep}), and (iii) multi-domain adaptation for regression under conditional shift (DARC; \citealt{taghiyarrenani2023multi}). Implementation details for all methods are provided in \autoref{app:imple_detail}.

\subsection{Simulation Studies}
\label{subsec:sim_studies}

\begin{figure*}[!ht]
	\centering
	\captionsetup[subfigure]{justification=centering}
	\begin{subfigure}[t]{0.48\linewidth}
		\centering
		\includegraphics[width=1\linewidth]{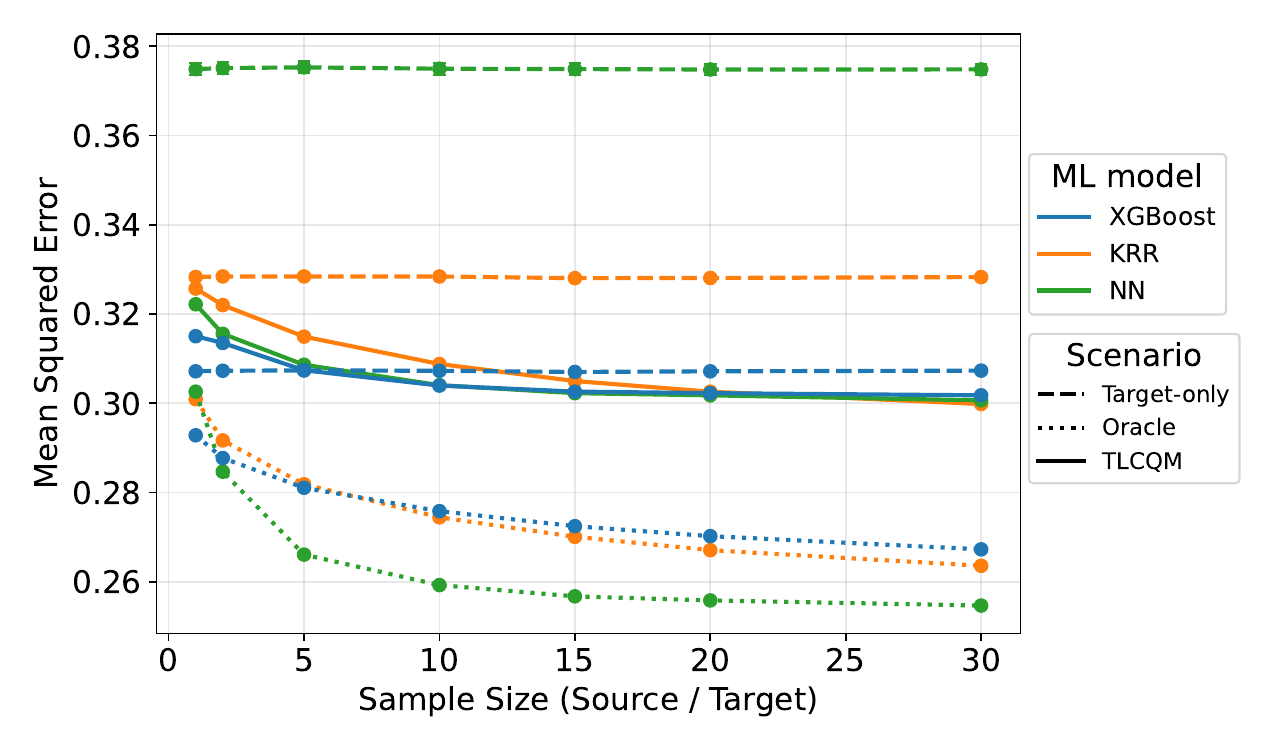}
		\caption{MSE as a function of the source-to-target sample size ratio for ML models under three data scenarios.}
	\end{subfigure}
	\hfil
	\begin{subfigure}[t]{0.5\linewidth}
		\centering		\includegraphics[width=1\linewidth]{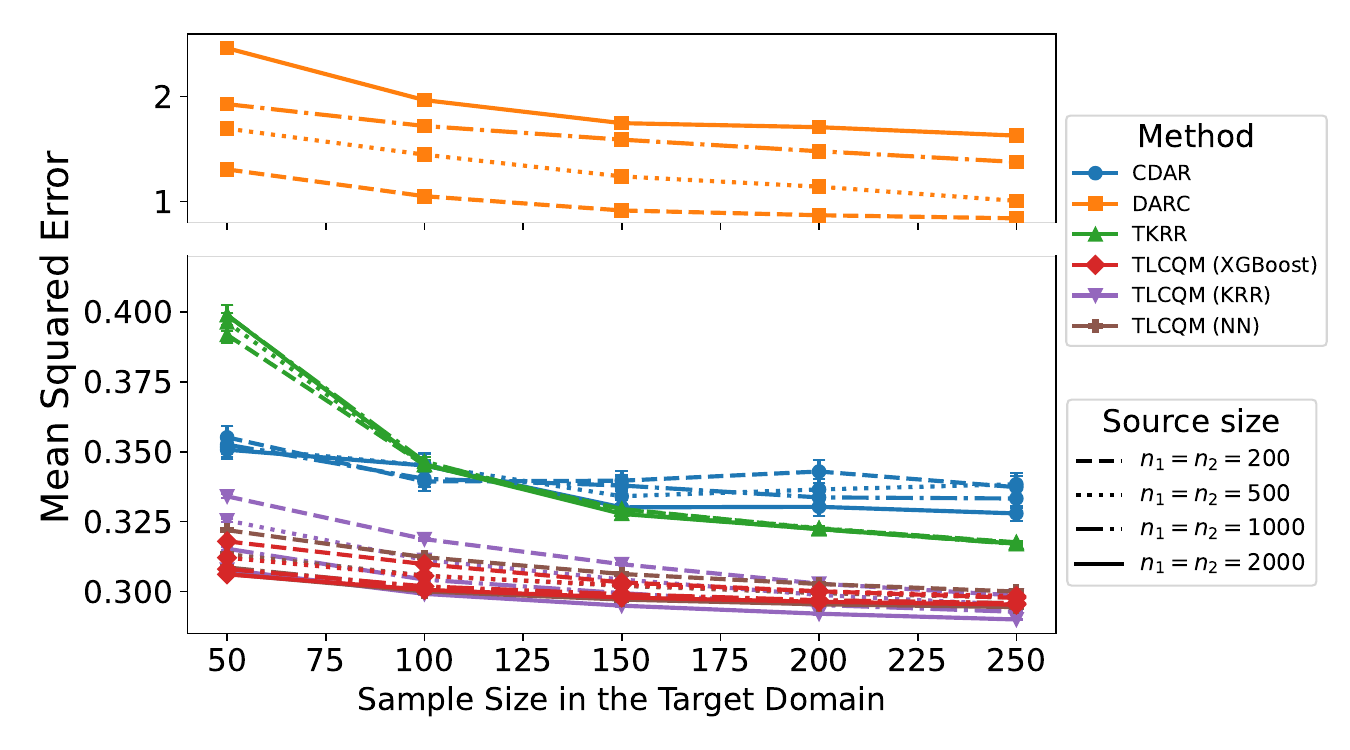}
		\caption{MSE as a function of the target-domain sample size for different transfer learning methods.}
	\end{subfigure}
	\caption{Prediction performance of different ML models under the simulated training scenarios and competing transfer learning methods.}
	\label{fig:sim1}
\end{figure*}

We consider a simulation setting in which both covariate shift and concept shift are present between the source and target domains. For the two candidate source domains, $Y^{(1)} = \sin\left(3\theta^T X^{(1)}\right) + 1+ \epsilon^{(1)}$ and $Y^{(2)} = 2\cos\left(3\theta^T X^{(2)}\right) + 1+ \epsilon^{(2)}$, where $\theta=\left(1,\frac{1}{2},...,\frac{1}{5}\right)^T\in \mathbb{R}^5$, $X^{(1)},X^{(2)}\sim \mathcal{N}(\bm{1}_5,\bm{I}_5)$, and $\epsilon^{(1)},\epsilon^{(2)}\sim \mathcal{N}\left(0,0.25\right)$. Here, $\bm{1}_5=(1,...,1)^T \in \mathbb{R}^5$ and $\bm{I}_5\in \mathbb{R}^{5\times 5}$ denote the vector of ones and the identity matrix, respectively. For the target domain, $Y^{(0)} = \frac{1}{3}\sin\left(3\theta^T X^{(0)}\right) -3 + \epsilon^{(0)}$, where $X^{(0)}\sim \mathcal{N}(\bm{0}_5,0.25\cdot \bm{I}_5)$ and $\epsilon^{(0)}\sim \mathcal{N}\left(0,0.25\right)$. 


We vary the source-to-target sample size ratio $\frac{n_k}{n_0}$ from 1 to 30 and the target-domain sample size $n_0$ from 50 to 250 in \autoref{fig:sim1}. To assess the effect of data augmentation by Algorithm~\ref{algo:CQM}, we consider three widely used ML models: XGBoost \citep{chen2016xgboost}, kernel ridge regression (KRR), and neural network (NN). We apply them under three training scenarios: target-only training, oracle training, and augmented training with TLCQM. The oracle data are drawn directly from the target-domain distribution and have the same sample size as the combined source and target datasets.

\autoref{fig:sim1}(a) reports the average mean squared errors (MSEs) across 1000 Monte Carlo repetitions at each source-to-target sample size ratio $\frac{n_k}{n_0}$. In this simulation setting, training on the TLCQM-augmented data generally improves prediction relative to target-only training, with larger gains when the source-to-target sample size ratio is high. The improvements are most pronounced for NN and KRR in the reported experiments.
We further compare our TLCQM framework with competing transfer learning methods that directly leverage both source and target data. As shown in \autoref{fig:sim1}(b), all ML models trained on TLCQM-augmented data outperform TKRR, CDAR, and DARC across the considered target sample sizes. Additional results for ML models trained on source-only data, naive pooling of source and target data, and simple target-only data augmentation are reported in \autoref{subsec:add_sim}, together with ablation studies in \autoref{subsec:abla} and examination of negative transfer in \autoref{subsec:diagnostic}. We also conduct evaluations on a public dataset in \autoref{subsec:real_public}.



\subsection{Case Study: Monthly Active User Prediction}

We apply our proposed TLCQM framework (Algorithm~\ref{algo:CQM}) to a real-world task involving the prediction of monthly active users (MAUs) for peer apps within Meta. The dataset consists of country-level MAU measurements collected over a single month, covering two device platforms and five distinct apps. To alleviate the heavy-tailed distribution, we apply a logarithmic transformation $u\mapsto \log(1+u)$ to all MAU measurements in this case study. Each app is treated as a separate data domain, with approximately 200 observations corresponding to different countries or platforms. The feature vector $X$ includes MAU statistics from other apps and a set of country-specific covariates, while the response $Y\in [0,\infty)$ represents the (log-transformed) country-level MAU for each app.

Our goal is to predict the country-level MAU for an app developed by a peer company, for which the third-party MAU estimates are available for a subset of countries. This application motivates the general shift setting \eqref{target_shift}, since both covariate and conditional country-level MAU distributions may vary across apps. To emulate this business scenario while enabling systematic evaluation, we use leave-one-app-out validation, where 4 out of the 5 apps are randomly selected as source domains with complete MAU information, and the remaining app is treated as the target domain, for which MAU statistics are observed only for the aforementioned subset of countries. The proposed TLCQM framework is then employed to augment or impute the missing MAU values for the remaining countries in the target app. To ensure the nonnegativity of the generated MAU values, we impose nonnegative constraints on the coefficients of 
the quantile matching estimator \eqref{quantile_ls} in our proposed framework, which can be interpreted as an implicit form of regularization tailored to this data application.

\begin{table*}[!t]
	\centering
	\setlength{\tabcolsep}{2.2pt}
	\scriptsize
	\begin{tabular}{c|ccc|ccc|ccc}
		\toprule
		& \multicolumn{3}{c}{\textbf{Target-Only}} & \multicolumn{3}{c}{\textbf{TLCQM}} & \multirow{2}{*}{\textbf{TKRR}} & \multirow{2}{*}{\textbf{CDAR}} & \multirow{2}{*}{\textbf{DARC}} \\
		& XGBoost & NN & KRR & XGBoost & NN & KRR & & & \\
		\midrule
		Platform I & 9.180 & 5.717 & 21.794 & \textbf{4.099 (0.014)} & \textbf{4.211 (0.012)} & \textbf{4.073 (0.014)} & 6.981 & 139.327 (1.093) & 18.220 (1.156)\\
		Platform II & 8.998 & 5.005 & 32.218 & \textbf{3.785 (0.065)} & \textbf{3.840 (0.046)} & \textbf{3.767 (0.068)} & 4.967 & 117.149 (1.194) & 11.542 (0.921)\\
		\bottomrule
	\end{tabular}
	\vspace{0.5mm}
	\caption{Average MSEs of different ML models trained on target-only and TLCQM-augmented data, together with those of competing transfer learning methods, for the MAU prediction task across two device platforms. Standard errors of the
		MSEs based on 30 Monte Carlo experiments are reported in parenthesis, where those values smaller than 0.001 are omitted for brevity.}
	\vspace{-3mm}
	\label{tab:MAU_res}
\end{table*}

The empirical results are summarized in \autoref{tab:MAU_res}, which reports the average MSEs for each ML model trained on data augmented by our proposed TLCQM framework. For comparison, we also include results from ML models trained on the target-only data as well as other competing transfer learning methods that leverage both source and target data, including TKRR, CDAR, and DARC that we mentioned earlier.
Since our proposed Algorithm~\ref{algo:CQM} involves generative model learning and synthetic data generation, their reported MSEs exhibit a small degree of Monte Carlo variability as in \autoref{tab:MAU_res}. Nevertheless, across all models and both platforms, the TLCQM-based approaches consistently outperforms the target-only ML models and other competing transfer learning methods, even after accounting for this additional randomness. These results demonstrate the robustness and practical relevance of the proposed framework in a business-critical application.

\section{Discussion}
\label{sec:discuss}

In summary, this work presents a transfer learning framework that combines source-domain conditional generators through a shared calibration estimated at the observed target covariates. Our theoretical and empirical results show that the framework can generate useful synthetic regression data for data-scarce target domains, while also revealing the risk of negative transfer when the target conditional distribution is poorly approximated by the calibrated family induced by the source generators. 
Although our analysis focuses on regression with univariate continuous responses, the conditional quantile calibration principle can be extended to other settings. For multivariate continuous responses, the quantile matching estimator in \eqref{quantile_ls} can be generalized via a weighted sum of CRPS objectives or a multivariate proper score, such as the energy score \citep{gneiting2007strictly}. For binary or categorical responses in classification tasks, CRPS coincides with the Brier score.
More broadly, the distributional calibration step in our proposed framework could be implemented using alternative metrics or operators, such as the sliced Wasserstein metric \citep{bonneel2015sliced}, Sinkhorn distance \citep{cuturi2013sinkhorn}, or maximum mean discrepancy \citep{cui2020calibrated}. A systematic investigation of these extensions is left for future work.

\section*{Acknowledgments}

The authors thank Xu Chen, Rizen Yamauchi, and other members of the Central Applied Science as well as Research Platform teams at Meta for their support and valuable feedback throughout this research.

\section*{Impact Statement}

This paper presents work whose goal is to advance the field of machine
learning. There are many potential societal consequences of our work, none of which we feel must be specifically highlighted.

\bibliography{TL_bib}
\bibliographystyle{plainnat}


\appendix
\newpage 

The first row in \autoref{tab:rad_complexity} is also known as the finite class lemma or Massart's lemma \citep{massart2000some}. By Sauer's Lemma, if the VC dimension $\mathrm{VC}(\mathcal{F})$ of a bounded real-valued function class $\mathcal{F}$ is less than $n$, then $\mathrm{Rad}_n(\mathcal{F}) = O\left(\sqrt{\frac{\mathrm{VC}(\mathcal{F}) \log(2n)}{n}}\right)$.

\begin{table}[ht]
	\centering
	\footnotesize
	\sloppy
	\setlength{\tabcolsep}{1pt}
	\begin{tabularx}{\textwidth}{CcC}
		\toprule
		\textbf{Function Class}
		& \textbf{Upper Bound}
		& \textbf{References} \\
		\midrule
		Finite hypothesis class $\mathcal F$ with $\sup_{x\in \mathcal{X}}|f(x)|\leq B_f$ for any $f\in \mathcal{F}$.
		& $\displaystyle O\left(B_f\sqrt{\frac{\log \left(2|\mathcal{F}|\right)}{n}}\right)$
		& Theorem 3.7 in \citet{mohri2018foundations}; see also \citet{massart2000some}. \\
		
		Linear predictor class 
		{\small$\mathcal{F}=\left\{x\mapsto \Upsilon^Tx: \norm{\Upsilon}_2\leq B_{\gamma},\norm{x}_2 < B_x\right\}$}.
		& $\displaystyle O\left(\frac{B_{\gamma} B_x}{\sqrt n}\right)$
		& Theorem 3 in \citet{kakade2008complexity}; see also Theorem 1 in \citet{awasthi2020rademacher}. \\
		
		Sparse linear predictor
		{\scriptsize$\mathcal{F}=\{x\mapsto \Upsilon^Tx:x\in\mathbb R^d,$\newline
		$\norm{\Upsilon}_0\leq s,\ \norm{\Upsilon}_2\leq B_\gamma,\ \norm{x}_2<B_x\}$.}
		& $\displaystyle O\left(B_\gamma B_x\sqrt{\frac{s\log d}{n}}\right)$
		& Section 3.1 in \citet{kakade2008complexity}; see also Lemma 7.1 in \citet{koltchinskii2011oracle}. \\
		
		Reproducing kernel Hilbert space $\mathcal{F}=\left\{f\in\mathcal{H}:\|f\|_{\mathcal H}\leq B_f\right\}$ with $\sup_x\sqrt{k(x,x)}\leq\kappa$
		& $\displaystyle O\left(\frac{B_f\kappa}{\sqrt n}\right)$
		& Lemma 22 in \citet{bartlett2002rademacher}. \\
		
		Binary decision tree class with $L$ nodes and $d$ real-valued features
		& {\scriptsize$\displaystyle O\left(\sqrt{\frac{L\log(2Ld)\log(2n)}{n}}\right)$}
		& Corollary 10 in \citet{leboeuf2020decision}. \\
		
		Neural network with the ReLU activation function, $W$ weights, $L$ layers, and outputs clipped to $[-B_f, B_f]$.
		& $\displaystyle O\left(
		\sqrt{\frac{WL\log(2W)\log(2n)}{n}}
		\right)$
		& Theorem 6 in \citet{bartlett2019nearly}; see also Theorem 5 in \citet{yin2019rademacher}, Theorem 14 in \citet{truong2022rademacher}, and Theorem 3.3 in \citet{bartlett2017spectrally}.  \\
		\bottomrule
	\end{tabularx}
\vspace{2mm}
	\caption{Upper bounds on the Rademacher complexity $\mathrm{Rad}_n(\mathcal{F})$ for common function classes}
	\label{tab:rad_complexity}
\end{table}

\section{Other Related Work}
\label{subsec:related_work}

{\bf Transfer learning:} Within the broad literature on transfer learning \citep{pan2009survey}, our work falls under supervised or transductive transfer learning for regression, also known as multi-domain adaptation \citep{mansour2008domain}. We therefore do not review related but distinct lines of work on transfer learning for classification \citep{cai2021transfer} or unsupervised domain adaptation \citep{ganin2016domain,long2015learning,chen2021representation}. 
Recent theoretical advances in supervised transfer learning for regression have focused on high-dimensional (generalized linear) models \citep{li2022transfer,tian2023transfer,zhou2024model,zhou2025doubly}, kernel ridge regression \citep{wang2023transfer,wang2023pseudo,lin2024smoothness,lin2025model}, and more general nonparametric settings \citep{cai2024transfer}. The generalized target shift scenario \eqref{target_shift} has also been studied through different angles, including location-scale conditional shift models \citep{zhang2013domain}, conditional embedding operators from the maximum mean discrepancy in deep learning \citep{liu2021deep}, outcome coarsening combined with representation learning \citep{wu2023multi}, pairwise similarity-preserving feature extraction to correct distributional discrepancies in $X | Y$ \citep{taghiyarrenani2023multi}, and semiparametric regression models with representation learning \citep{he2024representation}.
Our focus instead is conditional distribution calibration for data augmentation under the general shift setting \eqref{target_shift}.

{\bf Quantile matching:} Quantile matching \citep{sgouropoulos2015matching} has appeared in various contexts, including two-sample testing \citep{kosorok1999two}, parameter estimation \citep{dominicy2013method}, model diagnostics via quantile-quantile plots \citep{stuart2010matching}, and Bayesian methods \citep{nirwan2020bayesian}. More broadly, quantile matching can be interpreted as a special case of solving an optimal transport problem between univariate marginal distributions \citep{mallows1972note,villani2008optimal}. 

Finally, we emphasize that our work is distinct from transfer learning for quantile regression \citep{zhang2022transfer,qiao2024transfer,bai2024transfer,jin2024transfer}, which aims to improve estimation and inference for conditional quantiles in the target domain by leveraging source data. In contrast, our framework focuses on data augmentation through distributional alignment and can be used as a general preprocessing step to enhance a wide range of downstream learning tasks, including but not limited to quantile regression.

\section{Additional Implementation Details}
\label{app:imple_detail}

In this section, we document the practical considerations involved in implementing our proposed TLCQM framework and other experimental setups. 

\subsection{Computing Environment}
\label{app:compute_envir}

For parallelizing our Monte Carlo experiments, our code was executed on a computing clusters with 82 compute nodes of CPU cores in total. The required memory for each job is clearly stated in our code. 

\subsection{Setups for Engression and Conditional Quantile Matching}

For the engression method described in \autoref{subsec:engress}, we employ a neural network model with two hidden layers, each consisting of 100 neurons. The latent noise dimension is set to 5, the learning rate to 0.001, and the model is trained for 1000 epochs unless otherwise specified.

For the conditional calibration estimator in \autoref{subsec:quant_match}, we generate $M=3000$ responses from each fitted source conditional generative model at every observed target covariate. Unless otherwise stated, we minimize \eqref{quantile_ls} over
$$\Theta=\left\{\bm\beta\in\mathbb R^{K+1}:
\norm{\bm\beta}_1\leq 10\right\}.$$

For a fixed $\bm\beta$, we define
$\hat Z_{ij}(\bm\beta)=\bm\beta^T\hat{\bm V}_{ij}$, and let
$\hat Z_{i(1)}(\bm\beta)\leq\cdots\leq
\hat Z_{i(M)}(\bm\beta)$ denote the corresponding order statistics for
target observation $i$. The pairwise term in \eqref{quantile_ls} can then
be evaluated as
\[
\frac{1}{2M(M-1)}
\sum_{j\neq j'}
\left|\hat Z_{ij}(\bm\beta)-\hat Z_{ij'}(\bm\beta)\right|
=
\frac{1}{M(M-1)}
\sum_{r=1}^M(2r-M-1)\hat Z_{i(r)}(\bm\beta).
\]
Consequently, evaluating the empirical CRPS and a corresponding
subgradient costs $O\left(n_0M(K+\log M)\right)$. Crucially, the $M$
responses are sorted separately for each $i$ so that responses generated at different target covariates are neither compared nor pooled.

We initialize the optimization by regressing the observed target
responses on the Monte Carlo conditional means of the fitted source
models and projecting the resulting coefficient vector onto $\Theta$ as:
$$\tilde{\bm\beta} = \Pi_\Theta\left[ \argmin_{\bm\beta\in\mathbb R^{K+1}} \sum_{i=1}^{n_0} \left\{Y_i^{(0)} -\beta_0 -\sum_{k=1}^K\beta_k \left(\frac{1}{M}\sum_{j=1}^M\hat Y_{ij}^{(k)}\right)\right\}^2 \right],$$
where $\Pi_\Theta$ is the projection operator onto the set $\Theta$.
Because \eqref{quantile_ls} is generally nonconvex, we supplement this
least-squares initialization with 9 random feasible initializations
and retain the terminal solution having the smallest finite empirical
CRPS. For each initialization, we apply the proximal
difference-of-convex update \eqref{dca_update} with
$\gamma=10^{-4}$. At each outer iteration, the convex surrogate is
solved numerically using sequential least-squares programming (SLSQP) with at most 1000 inner iterations. A candidate solution is accepted only when it also does not increase the original empirical
CRPS. The algorithm runs for at most 1000 outer iterations and terminates when either
$$\norm{\bm\beta^{(r+1)}-\bm\beta^{(r)}}_2 \leq 10^{-8}\left(1+\norm{\bm\beta^{(r)}}_2\right)$$
or
$$\hat{\mathcal C}_{n_0,M}(\bm\beta^{(r)})- \hat{\mathcal C}_{n_0,M}(\bm\beta^{(r+1)}) \leq 10^{-8}\left[1+\left|\hat{\mathcal C}_{n_0,M}(\bm\beta^{(r+1)})\right|\right].$$

For squared-error prediction, we use the calibrated conditional mean
in \eqref{cond_mean_pred}, rather than an individual synthetic response, as the augmented label. The separate Monte Carlo size $M'$ used to approximate each source conditional mean is set to $M'=200$.

\subsection{Method for Density Ratio Estimation}

In Step 5 of Algorithm~\ref{algo:CQM}, we estimate the density ratio $w_k(x)=\frac{dP_X^{(0)}(x)}{dP_X^{(k)}(x)}$ for $k=1,...,K$ using kernel mean matching \citep{gretton2009covariate,gretton2012kernel}. Let $n_0$ and $n_k,k=1,...,K$ denote the sample sizes of the target and source domains, respectively. For each source domain $k$, given a positive definite kernel function $G:\mathcal{X}\times \mathcal{X}\to \mathbb{R}$, define the matrix $G^{(k)} \in \mathbb{R}^{n_k\times n_k}$ by $G_{ij}^{(k)}:= G(X_i^{(k)}, X_j^{(k)})$ and the vector $\kappa^{(k)}\in \mathbb{R}^{n_k}$ by $\kappa_i^{(k)}= \frac{n_k}{n_0} \sum_{j=1}^{n_0} G(X_i^{(k)}, X_j^{(0)})$. Kernel mean matching estimates the sample density-ratio vector $\zeta^{(k)} = \left(w_k(X_1^{(k)}),..., w_k(X_{n_k}^{(k)})\right)^T\in \mathbb{R}^{n_k}$ by solving
\begin{align*}
	\min_{\zeta^{(k)}\in \mathbb{R}^{n_k}}\quad&
		\frac{1}{n_k^2} \zeta^{(k)T} G^{(k)} \zeta^{(k)} - \frac{2}{n_k^2} \zeta^{(k)T} \kappa^{(k)}\\
	\text{subject to}\quad&
	0\leq\zeta_i^{(k)}\leq B_{\zeta},\qquad \left|\sum_{i=1}^{n_k} \zeta_i^{(k)} - n_k\right| \leq n_k\xi,
\end{align*}
where $\xi$ is often chosen to be $O\left(\frac{B_{\zeta}}{\sqrt{n_k}}\right)$.

\subsection{Hyperparameter Tuning for Standard Machine Learning Models and Comparative Methods}

When we apply the standard machine learning methods to datasets under various scenarios, their hyperparameters are tuned by five-fold cross-validation across a combination of the following candidate choices:
\begin{itemize}[leftmargin=*]
	\item {\bf XGBoost:} \texttt{learning\_rate}: $[0.001, 0.01, 0.1]$, \texttt{n\_estimators}: $[10, 50, 100]$, \texttt{max\_depth}: $[3, 5]$, \texttt{subsample}: $[0.8, 1.0]$, \texttt{colsample\_bytree}: $[0.8, 1.0]$.
	
	\item {\bf Kernel Ridge Regression (KRR):} \texttt{alpha} (penalty parameter): $\left\{3^r:r=-2,...,5\right\}\times \frac{0.1}{n}$, where $n$ is the sample size of the entire training set.
	
	\item {\bf Neural network:} \texttt{hidden\_layer\_sizes}: $[(10,), (50,), (100,)]$ and \texttt{alpha} ($L_2$ regularization strength): $[0.0001, 0.001, 0.01]$ with the activation function as ``ReLU'', the optimizer as Adam \citep{adam2014method}, and the number of epochs as 1000.
\end{itemize}
For transfer learning with kernel ridge regression (TKRR; \citealt{wang2023transfer}), we likewise use five-fold cross-validation for selecting the optimal penalty parameter \texttt{alpha} within the candidate set $\left[3^{-3}, 3^{-1},...,3^7\right]\times \frac{0.1}{n}$, where $n$ is the sample size of the entire training set. For deep transfer learning for conditional shift in regression (CDAR; \citealt{liu2021deep}), we follow the same setup in their paper, except that we replace the convolution architecture with a fully-connected neural network with $d\times 64 \times 16 \times 8$ hidden neurons. Finally, for multi-domain adaptation for regression under conditional shift (DARC; \citealt{taghiyarrenani2023multi}), we set the embedding or latent feature dimension across domains as 8, whose feature extraction model is a neural network with $(d+8)\times 100 \times 100$ hidden neurons.

\section{Additional Simulation Results}

This section provides supplementary simulation results for other baseline training strategies of the machine learning models, ablation studies for our proposed TLCQM framework, and diagnostic tests for negative transfer under other simulation setups.

\subsection{Other Baseline Training Strategies}
\label{subsec:add_sim}

We further compare the performance of XGBoost, KRR, and NN under several training strategies within the simulation framework described in \autoref{subsec:sim_studies}. Specifically, we compare models trained using (i) source-only data, (ii) naive pooling of source and target data, (iii) target-only data augmentation via the ``mixup'' technique \citep{zhang2018mixup}, and (iv) data augmentation via TLCQM. The ``mixup'' strategy generates additional target-domain observations by forming convex combinations of randomly selected pairs of target samples. Given two observations $(X_i^{(0)}, Y_i^{(0)})$ and $(X_j^{(0)}, Y_j^{(0)})$ randomly drawn from the original target-only training data, ``mixup'' creates the synthetic observation
$$\tilde{X}^{(0)} = \Lambda X_i^{(0)} + (1-\Lambda) X_j^{(0)}, \quad \tilde{Y}^{(0)} = \Lambda Y_i^{(0)} + (1-\Lambda) Y_j^{(0)},$$
where $\Lambda \sim \mathrm{Beta}(\tilde{\alpha},\tilde{\alpha})$ with tuning parameter $\tilde{\alpha} \in (0,1)$. In our experiments, we set the number of ``mixup''-generated samples to $5n_0$ with $n_0$ being the target-only sample size. For each machine learning model and data scenario, the tuning parameter $\tilde{\alpha}$ is selected from the candidate set $\left\{0.2, 0.4, 0.5, 0.6, 0.8\right\}$ to minimize the mean squared error (MSE).

\begin{table}[t]
	\centering
	\begin{threeparttable}
	\setlength{\tabcolsep}{3pt}
	\begin{tabular}{ccccccc}
		\toprule
		Ratio $\frac{n_k}{n_0}$ & Source 1 only & Source 2 only & Target only & Naive pooling & ``Mixup'' (best $\tilde{\alpha}$) & TLCQM \\
		\midrule
		\multicolumn{7}{c}{\textbf{XGBoost}} \\
		\midrule
		1  & 4.317 & 4.323 & \textbf{0.307} & 0.619 & 0.354 & 0.315 \\
		2  & 4.320 & 4.270 & \textbf{0.307} & 1.221 & 0.354 & 0.314 \\
		5  & 4.357 & 4.158 & \textbf{0.307} & 2.054 & 0.355 & 0.308 \\
		10 & 4.485 & 4.226 & 0.307 & 2.421 & 0.354 & \textbf{0.304} \\
		15 & 4.666 & 4.397 & 0.307 & 2.494 & 0.354 & \textbf{0.303} \\
		20 & 4.809 & 4.485 & 0.307 & 2.516 & 0.354 & \textbf{0.302} \\
		30 & 4.967 & 4.629 & 0.307 & 2.545 & 0.355 & \textbf{0.302} \\
		\midrule
		\multicolumn{7}{c}{\textbf{Kernel Ridge Regression (KRR)}} \\
		\midrule
		1  & 4.128 & 3.741 & 0.328 & 0.550 & 1.314 & \textbf{0.326} \\
		2  & 5.174 & 4.411 & 0.328 & 0.679 & 1.314 & \textbf{0.322} \\
		5  & 6.308 & 7.469 & 0.328 & 0.993 & 1.313 & \textbf{0.315} \\
		10 & 6.539 & 10.146 & 0.328 & 1.406 & 1.313 & \textbf{0.309} \\
		15 & 6.459 & 11.180 & 0.328 & 1.724 & 1.314 & \textbf{0.305} \\
		20 & 6.300 & 11.833 & 0.328 & 1.983 & 1.312 & \textbf{0.303} \\
		30 & 5.984 & 12.657 & 0.328 & 2.407 & 1.311 & \textbf{0.300} \\
		\midrule
		\multicolumn{7}{c}{\textbf{Neural Network (NN)}} \\
		\midrule
		1  & 5.212 & 6.102 & 0.375 & 0.628 & 0.417 & \textbf{0.322} \\
		2  & 5.009 & 5.787 & 0.375 & 0.721 & 0.418 & \textbf{0.316} \\
		5  & 4.907 & 6.366 & 0.375 & 1.016 & 0.418 & \textbf{0.309} \\
		10 & 4.901 & 7.354 & 0.375 & 1.385 & 0.418 & \textbf{0.304} \\
		15 & 4.908 & 7.730 & 0.375 & 1.710 & 0.417 & \textbf{0.302} \\
		20 & 4.847 & 7.839 & 0.375 & 2.013 & 0.417 & \textbf{0.302} \\
		30 & 4.738 & 7.789 & 0.375 & 2.498 & 0.417 & \textbf{0.301} \\
		\bottomrule
	\end{tabular}
	\end{threeparttable}
	\vspace{2mm}
	\caption{Prediction performance of machine learning models under different training strategies and source-to-target ratios. For each model and data scenario, the ``mixup'' column reports the best-performing choice of $\tilde{\alpha}$ in terms of MSE.}
	\label{tab:mse_add_sim}
\end{table}

\subsection{Ablation Studies}
\label{subsec:abla}

The experiments in \autoref{subsec:add_sim} already provide preliminary ablation evidence on the importance of proper transfer learning strategies, where training machine learning models on source-only data, target-only data, naively pooling all available data, or simple target-only data augmentation often leads to inferior performances. Here, we conduct more systematic ablation experiments under the simulation setup in \autoref{subsec:sim_studies} to clarify the contributions of the key components in our TLCQM framework. Specifically, we compare four different variants: (i) engression only (no calibration), (ii) conditional mean matching (MM), (iii) optimal transport (OT) calibration based on empirical cumulative distribution functions, and (iv) the proposed TLCQM framework. Notice that the calibration procedures in variants (ii-iv) are applied on top of the generated responses from the same fitted engression model.

To evaluate the performances of these four variants, we compute the MSEs of XGBoost, KRR, and NN trained on the generated or calibrated responses, as well as the Wasserstein distance between the generated or calibrated responses and the true responses in the target domain. Notably, smaller marginal Wasserstein distances imply better distributional alignment between the generated responses and the true target responses, which often benefits downstream machine learning tasks not only for point prediction but also for quantile regression and distributional learning.

As shown by the ablation results in \autoref{tab:abla_transfer}, the engression model alone is insufficient for effective transfer learning, indicating that the performance gains of our TLCQM framework are not merely due to a strong generative model. In contrast, all calibrated variants perform substantially better, confirming that distributional calibration is indispensable. Furthermore, simple mean matching and TLCQM achieve comparable MSEs for downstream machine learning models in this simulation. Nevertheless, TLCQM substantially outperforms both simple mean matching and OT calibration in terms of distributional alignment, as reflected by its much smaller reported Wasserstein distances to the true target responses than the other two variants. These ablation studies further highlight the integral design of our TLCQM framework, where both generative models and the quantile matching estimator play essential roles.

\begin{table}[t]
\centering
\begin{threeparttable}
	\setlength{\tabcolsep}{6pt}
	\begin{tabular}{cccccccc}
		\toprule
		Source-to-Target Ratio $\frac{n_k}{n_0}$ & 1 & 2 & 5 & 10 & 15 & 20 & 30 \\
		\midrule
		\multicolumn{8}{c}{\textbf{Mean Squared Error (XGBoost)}} \\
		\midrule
		Engression Only
		& 0.880 & 2.420 & 2.965 & 3.209 & 3.319 & 3.374 & 3.557 \\
		Conditional MM
		& \textbf{0.307} & \textbf{0.309} & \textbf{0.307}
		& 0.305 & 0.303 & 0.303 & 0.302 \\
		OT Calibration
		& 0.357 & 0.379 & 0.421 & 0.470 & 0.507 & 0.537 & 0.587 \\
		TLCQM
		& 0.315 & 0.314 & \textbf{0.307}
		& \textbf{0.304} & \textbf{0.303} & \textbf{0.302}
		& \textbf{0.302} \\
		\midrule
		\multicolumn{8}{c}{\textbf{Mean Squared Error (KRR)}} \\
		\midrule
		Engression Only
		& 0.757 & 0.959 & 1.369 & 1.816 & 2.156 & 2.424 & 2.821 \\
		Conditional MM
		& \textbf{0.319} & \textbf{0.318} & \textbf{0.314}
		& 0.310 & 0.306 & 0.304 & 0.301 \\
		OT Calibration
		& 0.371 & 0.394 & 0.450 & 0.514 & 0.557 & 0.586 & 0.625 \\
		TLCQM
		& 0.326 & 0.322 & 0.315
		& \textbf{0.309} & \textbf{0.305} & \textbf{0.303}
		& \textbf{0.300} \\
		\midrule
		\multicolumn{8}{c}{\textbf{Mean Squared Error (NN)}} \\
		\midrule
		Engression Only
		& 0.926 & 1.610 & 2.490 & 2.940 & 3.175 & 3.329 & 3.576 \\
		Conditional MM
		& \textbf{0.314} & \textbf{0.312} & \textbf{0.308}
		& 0.305 & 0.303 & 0.303 & 0.302 \\
		OT Calibration
		& 0.386 & 0.411 & 0.456 & 0.535 & 0.597 & 0.638 & 0.687 \\
		TLCQM
		& 0.322 & 0.316 & 0.309
		& \textbf{0.304} & \textbf{0.302} & \textbf{0.302}
		& \textbf{0.301} \\
		\midrule
		\multicolumn{8}{c}{\textbf{Wasserstein Distance}} \\
		\midrule
		Engression Only
		& 1.999 & 2.001 & 2.006 & 2.015 & 2.022 & 2.028 & 2.032 \\
		Conditional MM
		& 0.346 & 0.342 & 0.327 & 0.316 & 0.309 & 0.306 & 0.298 \\
		OT Calibration
		& 0.566 & 0.539 & 0.463 & 0.387 & 0.341 & 0.323 & 0.294 \\
		TLCQM
		& \textbf{0.0736} & \textbf{0.0530} & \textbf{0.0486}
		& \textbf{0.0479} & \textbf{0.0478} & \textbf{0.0479}
		& \textbf{0.0481} \\
		\bottomrule
	\end{tabular}
\end{threeparttable}
\vspace{2mm}
\caption{Results for ablation studies comparing engression only, conditional mean matching, OT calibration, and the proposed TLCQM framework across different source-to-target ratios.}
\label{tab:abla_transfer}
\end{table}

\subsection{Diagnostic Studies}
\label{subsec:diagnostic}

We further conduct two diagnostic simulations to distinguish marginal from conditional distributional alignment and to examine when TLCQM data augmentation may lead to negative transfer. In both settings below, $X=(X_1,...,X_5)^T\sim\mathcal{N}(\bm{0}_5,\bm{I}_5)$ in every domain, and the response noise is independently drawn from $\epsilon\sim \mathcal{N}(0,0.3^2)$. 
\begin{itemize}
	\item {\bf Reversed conditional relation:} For the two candidate source domains, $Y^{(1)} = -X_1 + \epsilon$ and $Y^{(2)}=X_2 +\epsilon$. For the target domain, $Y^{(0)} = X_1 + \epsilon$. Thus, the target conditional distribution belongs to the calibrated family in \eqref{quantile_ls} with $\bm\beta=(0,-1,0)^T$. Since $X_1$ is symmetric around 0, the target and first source domains have identical marginal response distributions, but their conditional distributions are reversed.
	
	\item {\bf Outside affine span:} For the two candidate source domains, $Y^{(1)} = \sin(\theta^TX) + \epsilon$ and $Y^{(2)}=\cos(\theta^TX) +\epsilon$, where $\theta=\left(1,\frac{1}{2},...,\frac{1}{5}\right)^T$. For the target domain, $Y^{(0)} = \sin(2\theta^T X) + \epsilon$. Hence, the target conditional mean lies outside the affine span generated by the two source conditional means, as specified in \eqref{quantile_ls} and further described in Remark~\ref{remark:cond_mean}.
\end{itemize}
We vary the target-domain sample size $n_0$ from 50 to 150 and adjust the source-to-target ratio $\frac{n_k}{n_0}$ between source and target sample sizes from 1 to 30 for $k=1,2$. For each downstream ML model (XGBoost, KRR, NN), we report the MSE differences 
$$\Delta\mathrm{MSE}
=\mathrm{MSE}_{\mathrm{TLCQM}}-\mathrm{MSE}_{\mathrm{Target-only}},$$
so negative values indicate favorable transfer. We also compute marginal and conditional Wasserstein-1 distances between the target distribution and either the first source distribution or the TLCQM-calibrated distribution. These Wasserstein-1 metrics allow us to assess the distributional discrepancies and intrinsic causes of negative transfer. At each of 3000 independent target covariates, we generate 500 responses from the known source and target conditional distributions.
The marginal Wasserstein distance pools responses across all test covariates, whereas the conditional Wasserstein distance is computed separately at each covariate and then averaged. 

\begin{figure}
	\centering
	\includegraphics[width=\linewidth]{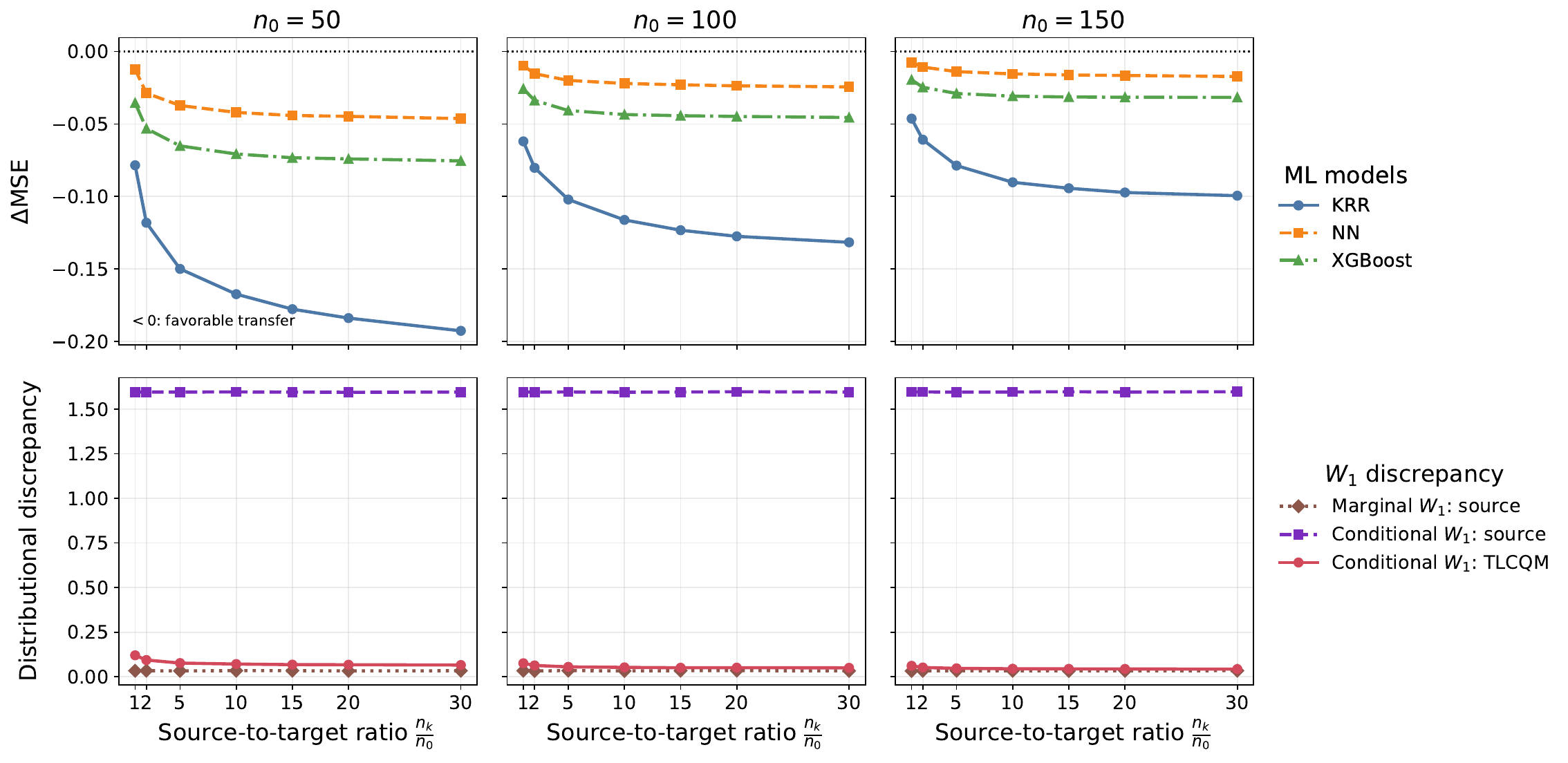}
	\caption{Simulation results under the ``\textbf{reversed conditional relation}'' scenario. The top row reports $\Delta\mathrm{MSE} =\mathrm{MSE}_{\mathrm{TLCQM}}-\mathrm{MSE}_{\mathrm{Target-only}}$; negative values indicate favorable transfer. The bottom row reports marginal and conditional Wasserstein-1 distances from the first source and TLCQM-calibrated distributions to the target one.}
	\label{fig:sim_diag_reverse}
\end{figure}

\begin{figure}
	\centering
	\includegraphics[width=\linewidth]{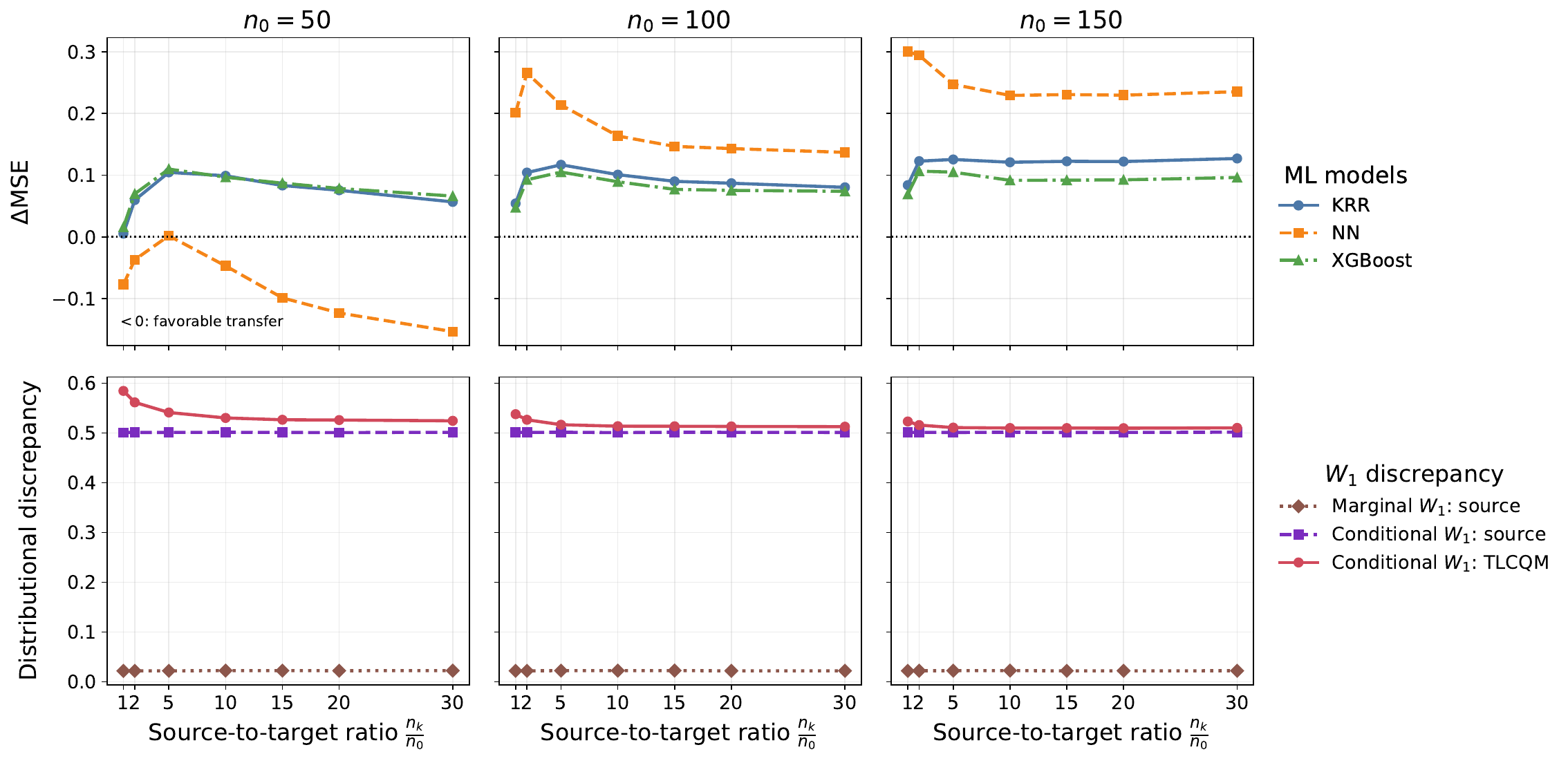}
	\caption{Simulation results under the ``\textbf{outside affine span}'' scenario.  The top row reports $\Delta\mathrm{MSE} =\mathrm{MSE}_{\mathrm{TLCQM}}-\mathrm{MSE}_{\mathrm{Target-only}}$; negative values indicate favorable transfer. The bottom row reports marginal and conditional Wasserstein-1 distances from the first source and TLCQM-calibrated distributions to the target one.}
	\label{fig:sim_diag_outside}
\end{figure}


The results are shown in \autoref{fig:sim_diag_reverse} and \autoref{fig:sim_diag_outside}. In the ``\textbf{reversed conditional relation}'' setting, the target and first source distributions have small marginal Wasserstein-1 distances but large conditional Wasserstein-1 distances across the considered values of $n_0$ and $n_k$.
This reveals that marginal alignment can conceal an incorrect covariate-response relationship.
In contrast, TLCQM recovers the conditional reversal and reduces the conditional Wasserstein-1 distances between TLCQM-calibrated and target responses to nearly zero. As a result, the downstream ML models trained on TLCQM-augmented data achieve lower MSEs than their target-only counterparts, indicating favorable transfer.
In the ``\textbf{outside affine span}'' setting, the marginal Wasserstein-1 distances are again small, but the conditional discrepancies of both source and TLCQM-calibrated distributions to the target one remain large. Since a shared linear calibration \eqref{quantile_ls_true} cannot eliminate this approximation bias, the downstream ML models trained on the TLCQM-augmented data often exhibit negative transfer. These results illustrate the bias-variance tradeoff in \autoref{thm:transfer_learn_excess_risk}: conditional quantile calibration \eqref{quantile_ls} can correct misleading marginal distributional alignment when the target distribution belongs to the calibrated family as in \eqref{quantile_ls_true}, but persistent transfer bias can dominate the generalization benefit of additional source data when the target conditional distribution lies outside this linear span.

\section{Experiments on Real Public Data}
\label{subsec:real_public}

To showcase the effectiveness of our proposed TLCQM framework on real-world data, we next evaluate it on a public dataset \texttt{Apartment} from the UCI machine learning repository \citep{kelly2019uci}. The response variable $Y$ is defined as the logarithm of the apartment rental price. The covariate vector $X$ comprises 10 features, including the numbers of bathrooms and bedrooms, the apartment size in square feet, and 7 binary indicators capturing the presence of photos, parking, storage, a gym, a pool, and whether cats and dogs are permitted.
The source domains consist of observations from three randomly selected states: \texttt{IL} ($n_1=1036$), \texttt{OH} ($n_2=4905$), and \texttt{WA} ($n_3=2612$). The target domain is the state \texttt{FL}, which contains 5775 observations in total. To construct the training data for the target domain, we randomly subsample $n_0\in \left\{100,200,300,500 \right\}$ observations from \texttt{FL} and reserve the remaining observations as the test set.

%
%

\begin{figure}[t]
	\centering
	\includegraphics[width=1\linewidth]{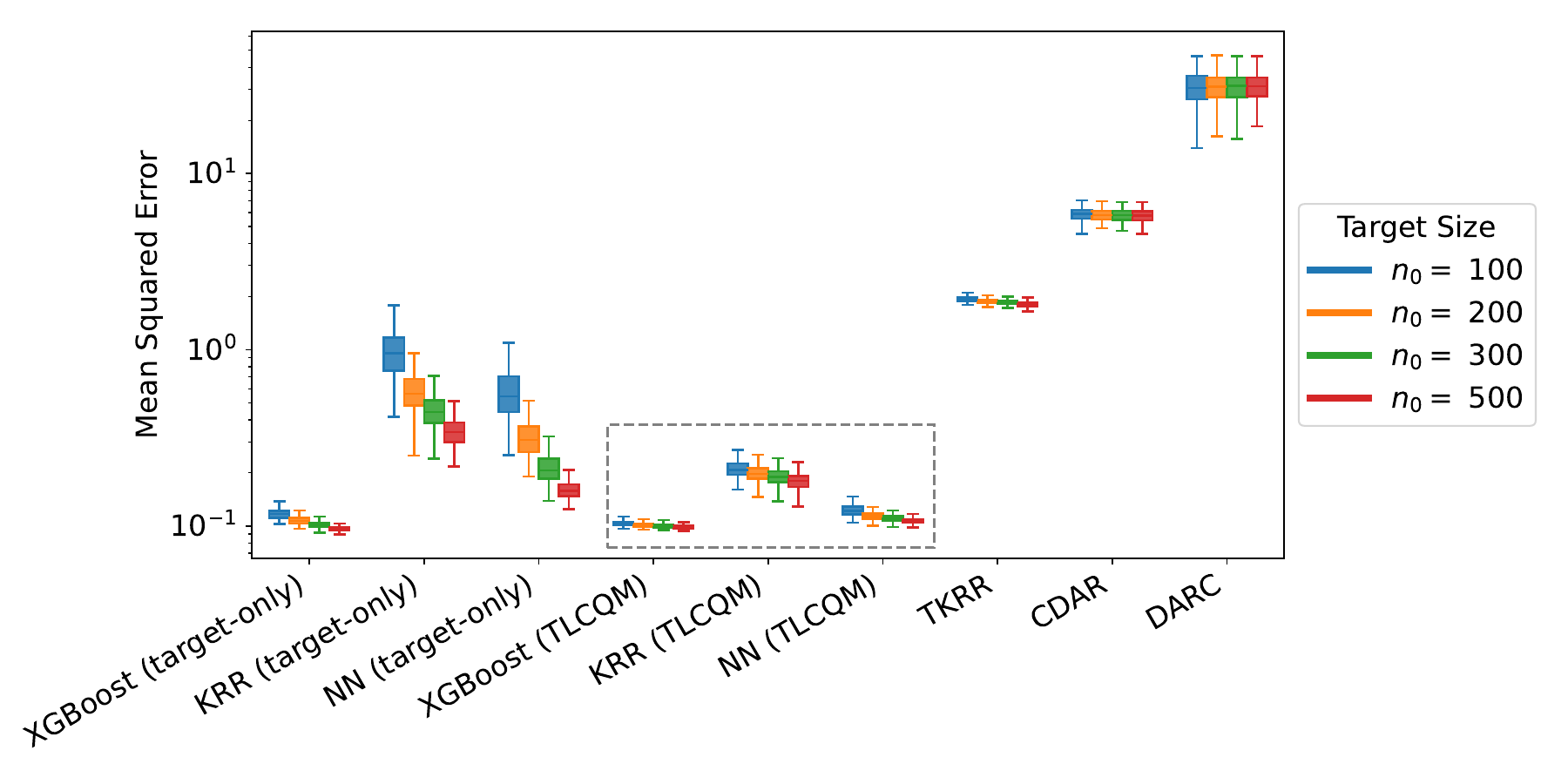}
	\caption{Performances of different machine learning models applied to target-only and TLCQM scenarios as well as other competing transfer learning methods on the ``Apartment'' data.}
	\label{fig:sim_apart}
\end{figure}

\begin{table}[t]
	\centering
	\begin{tabularx}{\linewidth}{C|CCCC}
		\toprule
		$n_0$ & $100$ & $200$ & $300$ & $500$ \\
		\midrule
		XGBoost & 11.99\% & 5.49\% & 1.70\% & -2.70\% \\
		KRR & 78.49\% & 66.00\% & 57.89\% & 47.81\% \\
		NN & 80.01\% & 65.86\% & 49.55\% & 34.10\% \\
		\bottomrule
	\end{tabularx}
	\vspace{2mm}
	\caption{Percentage reduction in MSE achieved by machine learning models trained on TLCQM-augmented data relative to their target-only counterparts on the ``Apartment'' data.}
	\label{tab:mse_improve_apart}
\end{table}

The results, summarized in \autoref{fig:sim_apart} and \autoref{tab:mse_improve_apart} based on 500 Monte Carlo replications, show that machine learning models trained under our proposed TLCQM framework consistently have smaller MSEs than their target-only counterparts as well as other competing transfer learning methods. The performance gains are particularly pronounced for NN, exceeding those observed for KRR and XGBoost.
This pattern is consistent with the mechanism of our TLCQM framework, which improves predictive performance primarily by enlarging the effective training sample size through the generation of well-calibrated synthetic data, thereby reducing estimation variance. Consequently, flexible learners whose performance is strongly sample-size dependent, such as NN and KRR, benefit more from the TLCQM framework. In contrast, tree-based models often exhibit stable performance in small-sample tabular settings \citep{shwartz2022tabular}, leaving less room for variance reduction to translate into substantial gains.

\section{Proofs for the Conditional Quantile Calibration Results}
\label{app:proof}

\subsection{Conditional CRPS Identity}
\label{app:cond_CRPS}

For any CDF $G$ with a finite first moment, we recall from \eqref{CRPS} that
$$\mathrm{CRPS}(G,y) =\int_{\mathbb{R}}\left[ G(t)-\mathds{1}\{y\leq t\}\right]^2 dt.$$
Conditioning on $X=x$, by Tonelli's theorem and using
$\mathbb{P}(Y^{(0)}\leq t\mid X=x)=F^{(0)}(t\mid x)$, we obtain that
\begin{align*}
\mathbb{E}\left[\mathrm{CRPS}(G(\cdot |X),Y^{(0)}) \big|X=x\right] &= \int_{\mathbb{R}} \mathbb{E}\left[\left(G(t|x) - \mathds{1}\{Y^{(0)}\leq t\}\right)^2 \Big|X=x\right] dt\\
&= \int_{\mathbb{R}} \left[G(t|x)^2 - 2F^{(0)}(t|x) \cdot G(t|x) + F^{(0)}(t|x)\right] dt\\
&= \int_{\mathbb{R}} \left[G(t|x) - F^{(0)}(t|x)\right]^2 dt + \int_{\mathbb{R}} F^{(0)}(t|x)\left[1 - F^{(0)}(t|x)\right]dt.
\end{align*}
Additionally, \eqref{crps_conditional_identity} can be proved by integrating over $X$ and noting that 
$$\mathbb{E}\left[\mathrm{CRPS}(F^{(0)}(\cdot|X), Y^{(0)}) |X=x\right] = \int_{\mathbb{R}} F^{(0)}(t|x)\left[1 - F^{(0)}(t|x)\right]dt.$$

\subsection{Proof of \autoref{thm:transfer_learn_excess_risk}}
\label{app:proof_excess_risk}

\begin{customthm}{3.4}
Suppose that Assumptions~\ref{assump:basic_reg} and \ref{assump:quantile_match} hold and that $\left(\hat{\bm{g}},\hat{\bm{\beta}},\hat{\bm{w}}\right)$ is fitted on auxiliary data independent of the observations used in \eqref{tranfer_learn_pred}. Assume further that $\ell(y,y') = (y-y')^2$ and $\mathcal{F}$ is a closed convex subset of $L_2(P_X^{(0)})$. Then, the transfer learning prediction function $\hat{f}^{(0,tl)}$ in \eqref{tranfer_learn_pred} under augmented responses as \eqref{cond_mean_pred} has its excess risk satisfying that, with probability at least $1-\delta$,
\begin{align*}
	R(\hat{f}^{(0,tl)}) - R(f^{(0)}) &\lesssim \underbrace{\mathrm{Rad}_N(\mathcal{F}) + \sqrt{\frac{\log(1/\delta)}{N}}}_{\textbf{Standard generalization error}} + \underbrace{\frac{1}{N} \sum_{k=1}^K \norm{\hat{w}_k-w_k}_1}_{\textbf{Importance weight error}} \\
	&\quad + \underbrace{\frac{(N-n_0)^2}{N^2}\left[\hat{\Delta}(\hat{\bm{\beta}}) + \frac{1}{M'}\right]}_{\textbf{Conditional transfer and Monte Carlo errors}}
\end{align*}
where $\norm{\hat{w}_k - w_k}_1 = \sum_{i=1}^{n_k}\left|\hat{w}_k(X_i^{(k)}) - w_k(X_i^{(k)}) \right|$ and $\delta \in (0,1)$. 

If $\mathcal{F}$ is not convex, $\ell$ is a general loss function satisfying Assumption~\ref{assump:basic_reg}(a) only, and the augmented labels are independent calibrated draws $\hat{\bm{\beta}}^T \hat{\bm{V}}_i^{(k)}$, then the last term becomes $\frac{N-n_0}{N} \sqrt{\hat{\Delta}(\hat{\bm{\beta}})}$.
\end{customthm}

\begin{proof}[Proof of \autoref{thm:transfer_learn_excess_risk}]
We first prove the excess risk bound under the squared loss $\ell(y,y') = (y-y')^2$ and $\mathcal{F}$ is a closed convex subset of $L_2(P_X^{(0)})$. Throughout the proof, we condition on $\left(\hat{\bm{g}},\hat{\bm{\beta}},\hat{\bm{w}}\right)$.

For \eqref{tranfer_learn_pred}, when $\hat{Z}_i^{(k)}=\hat{\mu}_{\hat{\bm{\beta}}}(X_i^{(k)})$ in \eqref{cond_mean_pred}, we denote the empirical objective by 
\begin{align*}
	\hat{R}_N^{\mu}(f; \hat{\bm w}) := \frac{1}{N} \left\{
	\sum_{i=1}^{n_0} \left(Y_i^{(0)}-f(X_i^{(0)})\right)^2 +\sum_{k=1}^K\sum_{i=1}^{n_k}\hat{w}_k(X_i^{(k)}) \left(\hat\mu_{\hat{\bm\beta}}(X_i^{(k)})-f(X_i^{(k)})\right)^2 \right\},
\end{align*}
and let $\hat R_N^{\mu}(f;\bm w)$ be the same objective with each $\hat{w}_k$ replaced by the true density ratio $w_k$. We also define
$$R_{\hat\mu}(f)=\mathbb E_{X\sim P_X^{(0)},\bm\eta}\left[\hat\mu_{\hat{\bm\beta}}(X)-f(X)\right]^2,
\quad R_A^{\mu}(f)=\frac{n_0}{N} \cdot R(f)+ \left(\frac{N-n_0}{N}\right) R_{\hat\mu}(f).$$
By the density-ratio identity under any function $h:\mathcal{X}\to \mathbb{R}$ as $\mathbb{E}_{P_X^{(k)}}\left[ w_k(X)\cdot h(X)\right]=\mathbb{E}_{P_X^{(0)}}\left[h(X)\right]$, we know that
$$\mathbb{E}\left[\hat{R}_N^{\mu}(f;\bm{w})\right] = R_A^{\mu}(f),$$
where the expectation is over $\left(X_i^{(0)}, Y_i^{(0)} \right)$ and $\left(X_i^{(k)}, Y_i^{(k)} \right), k=1,...,K$ as well as the independent noise vector $\bm{\eta}\sim P_{\eta}$.

Given the empirical minimizer $\hat f^{(0,tl)}$, we have that
\begin{align}
	R_A^{\mu}(\hat f^{(0,tl)})-R_A^{\mu}(f^{(0)})\notag &\leq
	\underbrace{\left[R_A^{\mu}(\hat f^{(0,tl)})-\hat R_N^{\mu}(\hat f^{(0,tl)};\bm w)\right] +\left[\hat R_N^{\mu}(f^{(0)};\bm w)-R_A^{\mu}(f^{(0)})\right]}_{\textbf{Term I}}\notag\\
	&\quad+
	\underbrace{\left[\hat R_N^{\mu}(\hat f^{(0,tl)};\bm w)-\hat R_N^{\mu}(\hat f^{(0,tl)};\hat{\bm w})\right]
		+\left[\hat R_N^{\mu}(f^{(0)};\hat{\bm w})-\hat R_N^{\mu}(f^{(0)};\bm w)\right]}_{\textbf{Term II}},
	\label{mean_erm_decomposition}
\end{align}
where the inequality follows from the fact that $\hat R_N^{\mu}(\hat f^{(0,tl)};\hat{\bm w})-
\hat R_N^{\mu}(f^{(0)};\hat{\bm w}) \leq 0$.

\noindent $\bullet$ \textbf{Term I:} Similar to Proposition~\ref{prop:target_only_excess}, we know that
\begin{align*}
\textbf{Term I} &\leq 2\sup_{f\in \mathcal{F}} \left|\hat{R}_N^{\mu}(f;\bm{w}) -R_A^{\mu}(f) \right| \\
&\leq 4B_{\ell}B_w \cdot \mathrm{Rad}_N(\mathcal{F}) + 2B_{\ell}B_w \sqrt{\frac{2\log(2/\delta)}{N}}
\end{align*}
with probability at least $1-\delta$.\\

\noindent $\bullet$ \textbf{Term II:} Uniformly over $f\in \mathcal{F}$, we know that
\begin{align*}
		\left|\hat R_N^{\mu}(f;\hat{\bm w})-\hat R_N^{\mu}(f;\bm w)\right| &= \left|\frac{1}{N}\sum_{k=1}^K \sum_{i=1}^{n_k} \left[\hat{w}_k(X_i^{(k)}) - w_k(X_i^{(k)}) \right]\left(\hat\mu_{\hat{\bm\beta}}(X_i^{(k)}) - f(X_i^{(k)})\right)^2\right| \\
	&\leq\frac{B_\ell}{N}\sum_{k=1}^K\sum_{i=1}^{n_k}
	\left|\hat w_k(X_i^{(k)})-w_k(X_i^{(k)})\right|\\
	&=\frac{B_\ell}{N} \sum_{k=1}^K\norm{\hat w_k-w_k}_1,
\end{align*}
where $\ell$ satisfies Assumption~\ref{assump:basic_reg}(a) and $\norm{\hat{w}_k - w_k}_1 = \sum_{i=1}^{n_k}\left|\hat{w}_k(X_i^{(k)}) - w_k(X_i^{(k)}) \right|$ for $k=1,...,K$. Thus,
\begin{align*}
	\textbf{Term II} \leq \frac{2B_{\ell}}{N} \sum_{k=1}^K \norm{\hat{w}_k-w_k}_1.
\end{align*}

Now, we connect $R_A^{\mu}(\hat f^{(0,tl)})-R_A^{\mu}(f^{(0)})$ with $R(\hat{f}^{(0,tl)}) - R(f^{(0)})$, which leads to the last ``\textbf{conditional transfer and Monte Carlo errors}'' term. Let $\mu^{(0)}(x)=\mathbb E(Y^{(0)}\mid X^{(0)}=x)$ and define
$$D_\mu^2=\mathbb{E}_{X,\bm\eta}
\left[\hat\mu_{\hat{\bm\beta}}(X)-\mu^{(0)}(X)\right]^2.$$
Under the squared loss $\ell(y,y') = (y-y')^2$, 
\begin{equation}
	\label{squared_risk_decomposition}
	R(f)=\mathbb E\left[\operatorname{Var}(Y^{(0)}\mid X^{(0)})\right]
	+\norm{\mu^{(0)}-f}_{L_2(P_X^{(0)})}^2.
\end{equation}
Thus, $f^{(0)} = \argmin_{f\in \mathcal{F}} R(f)$ is the metric projection of $\mu^{(0)}$ onto the closed convex set $\mathcal F$. From \eqref{squared_risk_decomposition}, it follows that
\begin{equation}
	\label{projection_excess_bound}
	\norm{f-f^{(0)}}_{L_2(P_X^{(0)})}^2
	\leq R(f)-R(f^{(0)}),
	\qquad \text{ for every }  f\in\mathcal F.
\end{equation}
A direct expansion of the squared losses yields that
\begin{align*}
	&R_A^{\mu}(f)-R_A^{\mu}(f^{(0)}) \\
	&= \frac{n_0}{N}\left[R(f)-R(f^{(0)})\right] + \left(\frac{N-n_0}{N}\right)\left\{\mathbb E_{X,\bm\eta}\left[\hat\mu_{\hat{\bm\beta}}(X)-f(X)\right]^2 - \mathbb E_{X,\bm\eta}\left[\hat\mu_{\hat{\bm\beta}}(X)-f^{(0)}(X)\right]^2 \right\}\\
	&=R(f)-R(f^{(0)}) - 2 \left(\frac{N-n_0}{N}\right)\mathbb E_{X,\bm\eta}\left\{\left[\hat\mu_{\hat{\bm\beta}}(X)-\mu^{(0)}(X)\right] \left[f(X)-f^{(0)}(X)\right]\right\}.
\end{align*}
Set $T=R(\hat f^{(0,tl)})-R(f^{(0)})$. Substituting $\hat{f}^{(0,tl)}$ into the preceding display, we obtain that
\begin{align*}
		T&= R(\hat f^{(0,tl)})-R(f^{(0)}) \\
		&= R_A^{\mu}(\hat f^{(0,tl)})-R_A^{\mu}(f^{(0)}) +  2 \left(\frac{N-n_0}{N}\right)\mathbb E_{X,\bm\eta}\left\{\left[\hat\mu_{\hat{\bm\beta}}(X)-\mu^{(0)}(X)\right] \left[\hat{f}^{(0,tl)}(X)-f^{(0)}(X)\right]\right\} \\
		&\stackrel{\text{(i)}}{\leq} R_A^{\mu}(\hat f^{(0,tl)})-R_A^{\mu}(f^{(0)}) + 2\left(\frac{N-n_0}{N}\right) D_{\mu} \norm{\hat f^{(0,tl)}-f^{(0)}}_{L_2(P_X^{(0)})}\\
		&\stackrel{\text{(ii)}}{\leq} R_A^{\mu}(\hat f^{(0,tl)})-R_A^{\mu}(f^{(0)}) + 2\left(\frac{N-n_0}{N}\right) D_{\mu} \sqrt{T}\\
		&\stackrel{\text{(iii)}}{\leq} R_A^{\mu}(\hat f^{(0,tl)})-R_A^{\mu}(f^{(0)}) + \frac{T}{2} + 2\left(\frac{N-n_0}{N}\right)^2 D_{\mu}^2,
\end{align*}
where (i) uses the Cauchy-Schwarz inequality, (ii) applies \eqref{projection_excess_bound}, and (iii) uses Young's inequality. Rearranging gives
\begin{align}
\label{excess_risk1}
\begin{split}
	T&= R(\hat f^{(0,tl)})-R(f^{(0)})\\
		&\leq 2\left[R_A^{\mu}(\hat f^{(0,tl)})-R_A^{\mu}(f^{(0)})\right] + 4\left(\frac{N-n_0}{N}\right)^2 D_{\mu}^2 \\
		&\stackrel{\text{(iv)}}{\leq} 8B_{\ell}B_w \cdot \mathrm{Rad}_N(\mathcal{F}) + 4B_{\ell}B_w \sqrt{\frac{2\log(2/\delta)}{N}} + \frac{4B_{\ell}}{N} \sum_{k=1}^K \norm{\hat{w}_k-w_k}_1 + 4\left(\frac{N-n_0}{N}\right)^2 D_{\mu}^2,
\end{split}
\end{align}
where (iv) follows from \eqref{mean_erm_decomposition} and our results for \textbf{Term I} and \textbf{Term II}.

To control $D_{\mu}^2$, we recall that $\hat{G}_{\hat{\bm{\beta}}}(t|x) = \mathbb{P}_{\bm{\eta}}\left[\hat{\bm{\beta}}^T \hat{\bm{V}}(x,\bm{\eta}) \leq t\big| \hat{\bm{\beta}},\hat{\bm{g}}\right]$. Let $\bar\mu_{\hat{\bm\beta}}(x) = \hat\beta_0+\sum_{k=1}^K\hat\beta_k \cdot \mathbb{E}_{\eta}\left[\hat g^{(k)}(x,\eta_k)\right]$. We know from \eqref{cond_mean_pred} that
\begin{align*}
	\mathbb{E}_{\bm\eta}\left[\hat\mu_{\hat{\bm\beta}}(x)-\mu^{(0)}(x)\right]^2 &=\left[\bar\mu_{\hat{\bm\beta}}(x)-\mu^{(0)}(x)\right]^2
	+\frac{1}{M'}\mathrm{Var}\left[\hat{\bm{\beta}}^T \hat{\bm{V}}(x,\bm{\eta}) \big| X=x\right] \\
	&= \left\{\int_{\mathbb{R}} \left[F^{(0)}(t| x)-\hat{G}_{\hat{\bm{\beta}}}(t|x) \right] dt\right\}^2 + \frac{1}{M'}\mathrm{Var}\left[\hat{\bm{\beta}}^T \hat{\bm{V}}(x,\bm{\eta}) \big| X=x\right] \\
	&\stackrel{\text{(v)}}{\leq} 2B_g \int_{\mathbb{R}} \left[F^{(0)}(t| x)-\hat{G}_{\hat{\bm{\beta}}}(t|x) \right]^2 dt + \frac{B_g^2}{M'},
\end{align*}
where (v) follows from Assumption~\ref{assump:quantile_match}. Taking the expectation over $X\sim P^{(0)}$, we obtain that
$$D_{\mu}^2 \leq 2B_g \hat{\Delta}(\hat{\bm{\beta}}) + \frac{B_g^2}{M'}.$$
Substituting this inequality into \eqref{excess_risk1} and noting that $M'$ can be chosen to be sufficiently large, we derive that
\begin{align*}
	&R(\hat{f}^{(0,tl)}) - R(f^{(0)}) \\
	&\lesssim 8B_{\ell}B_w \cdot \mathrm{Rad}_N(\mathcal{F}) + 4B_{\ell}B_w \sqrt{\frac{2\log(2/\delta)}{N}} + \frac{4B_{\ell}}{N} \sum_{k=1}^K \norm{\hat{w}_k-w_k}_1 \\
	&\quad + 4\left(\frac{N-n_0}{N}\right)^2 \left[2B_g \hat{\Delta}(\hat{\bm{\beta}}) + \frac{B_g^2}{M'}\right] \\
	&\lesssim \mathrm{Rad}_N(\mathcal{F}) + \sqrt{\frac{\log(1/\delta)}{N}} + \frac{1}{N} \sum_{k=1}^K \norm{\hat{w}_k-w_k}_1 + \left(\frac{N-n_0}{N}\right)^2 \left[\hat{\Delta}(\hat{\bm{\beta}}) + \frac{1}{M'}\right].
\end{align*}
This concludes the first part.\\

For the second part, suppose that $\mathcal{F}$ is not necessarily convex and that $\ell$ is a general loss function satisfying Assumption~\ref{assump:basic_reg}(a). When the augmented labels are sampled calibrated responses, we denote the empirical objective in \eqref{tranfer_learn_pred} by
$$\hat{R}_N^G(f;\hat{\bm{w}}) := \frac{1}{N} \left\{\sum_{i=1}^{n_0} \ell\left(Y_i^{(0)}, f(X_i^{(0)})\right) + \sum_{k=1}^K \sum_{i=1}^{n_k} \hat{w}_k(X_i^{(k)}) \cdot \ell\left(\hat{\bm{\beta}}^T \hat{\bm{V}}_i^{(k)}, f(X_i^{(k)})\right)\right\},$$
and let $\hat R_N^G(f;\bm w)$ be the same objective with the true density ratios. We also set
$$R_{\hat G}(f)=\mathbb E_{X\sim P_X^{(0)}}\left\{\mathbb E_{Z\sim\hat{G}_{\hat{\bm{\beta}}}(\cdot| X)}\left[\ell(Z,f(X)) \right] \right\},
\qquad R_A^G(f)=\frac{n_0}{N}\cdot R(f)+ \left(\frac{N-n_0}{N}\right)R_{\hat G}(f).$$
The density ratio identity again implies that $\mathbb{E}\left[\hat R_N^G(f;\bm w)\right]=R_A^G(f)$. Similar to \eqref{mean_erm_decomposition} and the related \textbf{Term I} and \textbf{Term II}, we know that
\begin{equation}
	\label{draw_augmented_erm_step}
	R_A^G(\hat f^{(0,tl)}) - R_A^G(f^{(0)}) \leq 4B_{\ell}B_w \cdot \mathrm{Rad}_N(\mathcal{F}) + 2B_{\ell}B_w \sqrt{\frac{2\log(2/\delta)}{N}} + \frac{2B_{\ell}}{N} \sum_{k=1}^K \norm{\hat{w}_k-w_k}_1.
\end{equation}
Now, by the definition of $R_A^G$, we know that
\begin{align}
	\label{excess_risk2}
	\begin{split}
			R(\hat f^{(0,tl)})-R(f^{(0)}) &= R_A^G(\hat{f}^{(0,tl)})-R_A^G(f^{(0)}) + \left(\frac{N-n_0}{N}\right)\left[\left(R-R_{\hat G}\right)(\hat{f}^{(0,tl)})- \left(R-R_{\hat G}\right)(f^{(0)})\right] \\
		&\leq 4B_{\ell}B_w \cdot \mathrm{Rad}_N(\mathcal{F}) + 2B_{\ell}B_w \sqrt{\frac{2\log(2/\delta)}{N}} + \frac{2B_{\ell}}{N} \sum_{k=1}^K \norm{\hat{w}_k-w_k}_1 \\
		&\quad + 2\left(\frac{N-n_0}{N}\right) \sup_{f\in \mathcal{F}}\left|R(f) - R_{\hat{G}}(f)\right|.
	\end{split}
\end{align}
Conditioning on $X=x$, we can couple $Y\sim F^{(0)}(\cdot|x)$ and $Z\sim \hat{G}_{\hat{\bm{\beta}}}(\cdot |x)$ by the common quantile coupling. By Assumption~\ref{assump:basic_reg}(a), we know that
\begin{align*}
	\left|\mathbb E\left[\ell(Y,f(x))| X=x\right]-\mathbb E\left[\ell(Z,f(x))| X=x\right]\right|
	&\leq B_{\ell} \int_{\mathbb{R}} \left|F^{(0)}(t| x)-\hat{G}_{\hat{\bm{\beta}}}(t|x)\right| dt \\
		&\stackrel{\text{(vi)}}{\leq} B_{\ell}\sqrt{2B_g \int_{\mathbb{R}} \left[F^{(0)}(t| x)-\hat{G}_{\hat{\bm{\beta}}}(t|x)\right]^2 dt},
\end{align*}
where (vi) follows from Assumption~\ref{assump:quantile_match} and the Cauchy-Schwarz inequality. Note that $\int_{\mathbb{R}} \left|F^{(0)}(t| x)-\hat{G}_{\hat{\bm{\beta}}}(t|x)\right| dt$ is the Wasserstein-1 distance between $F^{(0)}(\cdot|x)$ and $\hat{G}_{\hat{\bm{\beta}}}(\cdot |x)$.

Taking the expectation over $X\sim P_X^{(0)}$ and applying Jensen's inequality, we obtain that
$$\sup_{f\in \mathcal{F}}\left|R(f) - R_{\hat{G}}(f)\right| \leq  B_{\ell}\sqrt{2B_g \hat{\Delta}(\hat{\bm{\beta}})}.$$
Substituting this result into \eqref{excess_risk2}, we obtain that
\begin{align*}
		R(\hat f^{(0,tl)})-R(f^{(0)}) &\leq 4B_{\ell}B_w \cdot \mathrm{Rad}_N(\mathcal{F}) + 2B_{\ell}B_w \sqrt{\frac{2\log(2/\delta)}{N}} + \frac{2B_{\ell}}{N} \sum_{k=1}^K \norm{\hat{w}_k-w_k}_1 \\
		&\quad + 2B_{\ell}\left(\frac{N-n_0}{N}\right) \sqrt{2B_g \hat{\Delta}(\hat{\bm{\beta}})} \\
		&\lesssim \mathrm{Rad}_N(\mathcal{F}) + \sqrt{\frac{\log(1/\delta)}{N}} + \frac{1}{N} \sum_{k=1}^K \norm{\hat{w}_k-w_k}_1 + \left(\frac{N-n_0}{N}\right)\sqrt{B_g\hat{\Delta}(\hat{\bm{\beta}})}.
\end{align*}
This completes the proof.
\end{proof}

\subsection{Proof of Theorem~\ref{thm:conditional_calibration_rate}}
\label{app:proof_cond_cal_err}

\begin{customthm}{3.5}
Let $\delta \in (0,1)$ and $\norm{\hat{g}^{(k)} - g^{(k)}}_{\infty} =\sup_{x\in \mathcal{X}} \mathbb{E}_{\eta}\left|\hat{g}^{(k)}(x,\eta) - g^{(k)}(x,\eta)\right|$ for $k=1,...,K$. Suppose that Assumption~\ref{assump:quantile_match} holds and that $\hat{\bm{g}}$ is independent of the target observations used in \eqref{quantile_ls}. If $\hat{\mathcal C}_{n_0,M}(\hat{\bm{\beta}}) \leq \inf_{\bm\beta\in\Theta} \hat{\mathcal{C}}_{n_0,M}(\bm{\beta}) + \epsilon_{\rm opt}$ in \eqref{quantile_ls}, then with probability at least $1-\delta$,
\begin{align*}
	\hat{\Delta}(\hat{\bm{\beta}}) - \Delta(\bm{\beta}_*) &\lesssim \sum_{k=1}^K \norm{\hat{g}^{(k)} - g^{(k)}}_{\infty} + \sqrt{\frac{\left(\log n_0 + \log(1/\delta)\right) \left[\Delta(\bm{\beta}_*) + \sum_{k=1}^K \norm{\hat{g}^{(k)} - g^{(k)}}_{\infty}\right]}{n_0}} \\
	&\quad + \frac{\log n_0 + \log(1/\delta)}{n_0} + \epsilon_{\rm opt}
\end{align*}
up to some Monte Carlo approximation errors of order $O\left(\frac{1}{\sqrt{M}}\right)$.
\end{customthm}

\begin{proof}[Proof of \autoref{thm:conditional_calibration_rate}]
From \eqref{quantile_ls}, let
$$\hat{\mathcal{C}}_{i,M}(\bm{\beta}) = \frac{1}{M}\sum_{j=1}^M \left|Y_i^{(0)}- \bm{\beta}^T\hat{\bm V}_{ij}\right| -\frac{1}{2M(M-1)}\sum_{j, j'=1}^M \left|\bm{\beta}^T\hat{\bm V}_{ij} -\bm{\beta}^T\hat{\bm V}_{ij'}\right|.$$
Conditioning on $\hat{\bm g}$, we know from \eqref{estimated_conditional_crps} or \autoref{app:cond_CRPS} and \eqref{CRPS2} that
$$\hat{\Delta}(\bm{\beta}) = \mathbb{E}\left[\hat{\mathcal{C}}_{i,M}(\bm{\beta}) \,\big|\, \hat{\bm{g}}\right] - \mathcal{C}_0$$
for any fixed $\bm{\beta}\in \Theta$. Then, we define
\begin{align*}
	\hat{D}_i(\bm{\beta}) &= \hat{\mathcal{C}}_{i,M}(\bm{\beta}) - \mathrm{CRPS}\left(\hat G_{\bm\beta}(\cdot|X_i^{(0)}),Y_i^{(0)}\right) \\
	&\quad + \mathrm{CRPS}\left(\hat G_{\bm\beta}(\cdot|X_i^{(0)}),Y_i^{(0)}\right) - \mathrm{CRPS}\left(F^{(0)}(\cdot|X_i^{(0)}), Y_i^{(0)}\right)
\end{align*}
Equivalently, $\hat D_i(\bm\beta)=\hat{\mathcal C}_{i,M}(\bm\beta)-\mathrm{CRPS}(F^{(0)}(\cdot|X_i^{(0)}),Y_i^{(0)})$.
In addition, we have that
\begin{align*}
	&\mathrm{CRPS}\left(\hat G_{\bm\beta}(\cdot|X),Y\right)- \mathrm{CRPS}\left(F^{(0)}(\cdot|X),Y\right)\\
	&=\int_{\mathbb{R}} \left[\hat G_{\bm\beta}(t|X) - F^{(0)}(t|X)\right]\left[\hat G_{\bm\beta}(t|X)+F^{(0)}(t|X)- 2\mathds{1}\{Y\leq t\}\right]dt.
\end{align*}
By the Cauchy-Schwarz inequality and Assumption~\ref{assump:quantile_match}, we obtain
\begin{align*}
		&\mathbb{E}\left\{\left[\mathrm{CRPS}\left(\hat G_{\bm\beta}(\cdot|X),Y\right)- \mathrm{CRPS}\left(F^{(0)}(\cdot|X),Y\right)\right]^2 \Big|\, \hat{\bm{g}} \right\} \\
		&\stackrel{\text{(i)}}{\leq} 8B_g\,\mathbb{E}\left\{\int_{\mathbb{R}} \left[\hat G_{\bm\beta}(t|X) - F^{(0)}(t|X)\right]^2 dt \,\Big|\, \hat{\bm g}\right\}\\
		&=8 B_g \hat{\Delta}(\bm{\beta}),
\end{align*}
where (i) uses the fact that $Y$ is bounded by $B_g$, $\left|\hat{G}_{\bm\beta}(t|X)\right| \leq 1$, and $\left|F^{(0)}(t|X)\right| \leq 1$. Moreover, by the Monte Carlo average and U-statistic structure,
$$\mathbb{E}\left[\hat{\mathcal{C}}_{i,M}(\bm{\beta}) \,\big|\, \hat{\bm{g}}, X_i^{(0)},Y_i^{(0)} \right] = \mathrm{CRPS}\left(\hat G_{\bm\beta}(\cdot|X_i^{(0)}),Y_i^{(0)}\right),$$
and $\mathrm{Var}\left[\hat{\mathcal{C}}_{i,M}(\bm{\beta}) \big|\hat{\bm{g}}, X_i^{(0)},Y_i^{(0)} \right]$ is bounded by $\frac{C_1 B_g^2}{M}$ for a universal constant $C_1>0$. Therefore,
\begin{align*}
		\mathrm{Var}\left[\hat{D}_i(\bm{\beta}) \mid \hat{\bm{g}}\right] \leq C_2\left[\frac{B_g^2}{M} + B_g \hat{\Delta}(\bm{\beta})\right]
\end{align*}
for some universal constant $C_2>0$. Furthermore, under Assumption~\ref{assump:quantile_match},
$$|\hat D_i(\bm\beta)-\hat D_i(\bm\beta')|\leq 2B_g\norm{\bm\beta-\bm\beta'}_1$$
for any $\bm{\beta},\bm{\beta}'\in \Theta$, and the same Lipschitz bound holds for $\hat\Delta$. Since the $L_1$-ball $\Theta$ has an $\epsilon$-net of cardinality at most $\left(1+\frac{2B_{\beta}}{\epsilon}\right)^{K+1}$, Bernstein's inequality gives
$$\mathbb{P}\left(\left|\frac{1}{n_0}\sum_{i=1}^{n_0}\hat D_i(\bm\beta)-\hat\Delta(\bm\beta)\right|\geq C_3B_g \left\{\sqrt{\left[\frac{\hat\Delta(\bm\beta)}{B_g}+\frac{1}{M}\right]\frac{t}{n_0}}+\frac{t}{n_0}\right\}\right) \leq 2e^{-t}$$
for any $\bm{\beta}\in \Theta$. We now apply a union bound to a $\left(\frac{1}{n_0}\right)$-net and extend the result by Lipschitz continuity. Let $t_0(n_0,\delta)=(K+1)\log\left(1 + 2B_{\beta}n_0\right) + \log(2/\delta)$. Then, with probability at least $1-\delta$, the following inequality holds simultaneously for every $\bm\beta\in\Theta$:
\begin{align}
	\label{uniform_localized_crps}
	\begin{split}
		\left|\frac{1}{n_0}\sum_{i=1}^{n_0}\hat D_i(\bm\beta)-\hat\Delta(\bm\beta)\right|&\leq C_3B_g \left\{\sqrt{\left[\frac{\hat\Delta(\bm\beta)}{B_g}+\frac{1}{M}\right]\frac{t_0(n_0,\delta)}{n_0}} +\frac{t_0(n_0,\delta)}{n_0}\right\} \\
		&= C_3B_g \left\{\sqrt{\left[\frac{\hat\Delta(\bm\beta)}{B_g}+\frac{1}{M}\right]a_{n_0}(\delta)} +a_{n_0}(\delta)\right\}
	\end{split}
\end{align}
where $a_{n_0}(\delta) = \frac{t_0(n_0,\delta)}{n_0}$ and $C_3>0$ is an absolute constant.

Now, let $\bar{\bm\beta}\in\argmin_{\bm\beta\in\Theta}\hat\Delta(\bm\beta)$, which exists by compactness of $\Theta$. The empirical optimality of $\hat{\bm\beta}$ and \eqref{uniform_localized_crps}, together with Young's inequality to absorb the deviation term involving $\hat\Delta(\hat{\bm\beta})$, imply that
\begin{align}
	\label{localized_crps_oracle}
		\hat\Delta(\hat{\bm\beta})\leq\inf_{\bm\beta\in\Theta}\hat\Delta(\bm\beta)+CB_g\left\{\sqrt{\left[\frac{\inf_{\bm\beta\in\Theta}\hat\Delta(\bm\beta)}{B_g}+\frac{1}{M}\right]a_{n_0}(\delta)}+a_{n_0}(\delta)\right\}+C_4\epsilon_{\rm opt}
\end{align}
with probability at least $1-\delta$, where $C_4>0$ is an absolute constant.

Finally, we couple $\hat{g}^{(k)}$ and $g^{(k)}$ with the same noise variables for $k=1,...,K$ and write $\hat Z_{\bm\beta}-Z_{\bm\beta}=\sum_{k=1}^K\beta_k\left[\hat g^{(k)}(X,\eta_k) -g^{(k)}(X,\eta_k)\right]$. By the triangle inequality,
\begin{align*}
	&\left|\mathbb E\left[\mathrm{CRPS}\left(\hat G_{\bm\beta}(\cdot|X),Y\right)- \mathrm{CRPS}\left( G_{\bm\beta}(\cdot|X),Y\right)\right]\right| \\
	&\leq \mathbb{E}\left|\mathbb{E}_{\bm{\eta}}|Y- Z_{\bm{\beta}}| - \mathbb{E}_{\bm{\eta}}|Y- \hat{Z}_{\bm{\beta}}| -\frac{1}{2}\mathbb{E}_{\bm{\eta}}|Z_{\bm{\beta}} - Z_{\bm{\beta}}'| + \frac{1}{2}\mathbb{E}_{\bm{\eta}}|\hat{Z}_{\bm{\beta}} - \hat{Z}_{\bm{\beta}}'| \right| \\
	&\leq 2B_{\beta} \sum_{k=1}^K \sup_{x\in \mathcal{X}} \mathbb{E}_{\eta}\left|\hat{g}^{(k)}(x,\eta) - g^{(k)}(x,\eta)\right|\\
	&= 2B_{\beta} \sum_{k=1}^K \norm{\hat{g}^{(k)} - g^{(k)}}_{\infty}
\end{align*}
uniformly over $\bm{\beta}\in \Theta$. It follows that
$$\inf_{\bm\beta\in\Theta} \hat{\Delta}(\bm{\beta})\leq \underbrace{\inf_{\bm\beta\in\Theta} \Delta(\bm{\beta})}_{=\Delta(\bm{\beta}_*)} + 2B_{\beta} \sum_{k=1}^K \norm{\hat{g}^{(k)} - g^{(k)}}_{\infty}.$$

Substituting this inequality into \eqref{localized_crps_oracle} and absorbing constants that depend only on $B_g,B_{\beta}$ prove that
\begin{align*}
	\hat{\Delta}(\hat{\bm{\beta}}) - \Delta(\bm{\beta}_*) &\lesssim \sum_{k=1}^K \norm{\hat{g}^{(k)} - g^{(k)}}_{\infty} + \sqrt{\frac{\left(\log n_0 + \log(1/\delta)\right) \left[\Delta(\bm{\beta}_*) + \sum_{k=1}^K \norm{\hat{g}^{(k)} - g^{(k)}}_{\infty}\right]}{n_0}} \\
	&\quad + \frac{\log n_0 + \log(1/\delta)}{n_0} + \epsilon_{\rm opt}
\end{align*}
up to some Monte Carlo approximation errors $O\left(\frac{1}{\sqrt{M}}\right)$.
\end{proof}


\newpage
\section*{NeurIPS Paper Checklist}

\begin{enumerate}

\item {\bf Claims}
    \item[] Question: Do the main claims made in the abstract and introduction accurately reflect the paper's contributions and scope?
    \item[] Answer: \answerYes{} 
    \item[] Justification: The abstract and introduction state the general distribution shift setting in \eqref{target_shift}, qualify the empirical claims by the size of the transfer bias, and point to the assumptions and theoretical guarantees of our proposed transfer learning framework in \autoref{sec:theory}.

\item {\bf Limitations}
    \item[] Question: Does the paper discuss the limitations of the work performed by the authors?
    \item[] Answer: \answerYes{} 
    \item[] Justification: Our work focuses on transfer regression learning with a univariate response variable. The proposed method relies on a shared linear calibration and may suffer negative transfer when its conditional family provides a poor approximation to the target distribution; see Remark~\ref{remark:cond_mean}. We examine this limitation empirically in \autoref{subsec:diagnostic}. Our theoretical results also require boundedness of the response variable and, for the downstream ERM bound, sample splitting or cross-fitting. We further discuss connections to related work in \autoref{subsec:related_work}, extensions to classification and multivariate responses in \autoref{sec:discuss}, and computational complexity in \autoref{app:imple_detail}.

\item {\bf Theory assumptions and proofs}
    \item[] Question: For each theoretical result, does the paper provide the full set of assumptions and a complete (and correct) proof?
    \item[] Answer: \answerYes{} 
    \item[] Justification: Assumptions are stated in \autoref{sec:theory} before each result, and complete proofs are provided in \autoref{app:proof}.

    \item {\bf Experimental result reproducibility}
    \item[] Question: Does the paper fully disclose all the information needed to reproduce the main experimental results of the paper to the extent that it affects the main claims and/or conclusions of the paper (regardless of whether the code and data are provided or not)?
    \item[] Answer: \answerYes{} 
    \item[] Justification: We provide code that reproduces our experiments in simulation studies and applications to a public dataset, along with implementation details in \autoref{app:imple_detail}.
    \item[] Guidelines:
    \begin{itemize}
        \item The answer \answerNA{} means that the paper does not include experiments.
        \item If the paper includes experiments, a \answerNo{} answer to this question will not be perceived well by the reviewers: Making the paper reproducible is important, regardless of whether the code and data are provided or not.
        \item If the contribution is a dataset and\slash or model, the authors should describe the steps taken to make their results reproducible or verifiable. 
        \item Depending on the contribution, reproducibility can be accomplished in various ways. For example, if the contribution is a novel architecture, describing the architecture fully might suffice, or if the contribution is a specific model and empirical evaluation, it may be necessary to either make it possible for others to replicate the model with the same dataset, or provide access to the model. In general, releasing code and data is often one good way to accomplish this, but reproducibility can also be provided via detailed instructions for how to replicate the results, access to a hosted model (e.g., in the case of a large language model), releasing of a model checkpoint, or other means that are appropriate to the research performed.
        \item While NeurIPS does not require releasing code, the conference does require all submissions to provide some reasonable avenue for reproducibility, which may depend on the nature of the contribution. For example
        \begin{enumerate}
            \item If the contribution is primarily a new algorithm, the paper should make it clear how to reproduce that algorithm.
            \item If the contribution is primarily a new model architecture, the paper should describe the architecture clearly and fully.
            \item If the contribution is a new model (e.g., a large language model), then there should either be a way to access this model for reproducing the results or a way to reproduce the model (e.g., with an open-source dataset or instructions for how to construct the dataset).
            \item We recognize that reproducibility may be tricky in some cases, in which case authors are welcome to describe the particular way they provide for reproducibility. In the case of closed-source models, it may be that access to the model is limited in some way (e.g., to registered users), but it should be possible for other researchers to have some path to reproducing or verifying the results.
        \end{enumerate}
    \end{itemize}

\item {\bf Open access to data and code}
    \item[] Question: Does the paper provide open access to the data and code, with sufficient instructions to faithfully reproduce the main experimental results, as described in supplemental material?
    \item[] Answer: \answerNo{} 
    \item[] Justification: Code is provided in the supplement and on GitHub, and the public UCI dataset is openly available. The proprietary MAU data cannot be released.

\item {\bf Experimental setting/details}
    \item[] Question: Does the paper specify all the training and test details (e.g., data splits, hyperparameters, how they were chosen, type of optimizer) necessary to understand the results?
    \item[] Answer: \answerYes{} 
    \item[] Justification: The experimental details are clearly delineated in \autoref{app:imple_detail}.

\item {\bf Experiment statistical significance}
    \item[] Question: Does the paper report error bars suitably and correctly defined or other appropriate information about the statistical significance of the experiments?
    \item[] Answer: \answerYes{} 
    \item[] Justification: The figures and tables report Monte Carlo standard errors across independent repetitions, and values below 0.001 are omitted from the MAU table.

\item {\bf Experiments compute resources}
    \item[] Question: For each experiment, does the paper provide sufficient information on the computer resources (type of compute workers, memory, time of execution) needed to reproduce the experiments?
    \item[] Answer: \answerYes{} 
    \item[] Justification: \autoref{app:compute_envir} reports the compute type and cluster size, while the code specifies the memory allocated to each job.
    
\item {\bf Code of ethics}
    \item[] Question: Does the research conducted in the paper conform, in every respect, with the NeurIPS Code of Ethics \url{https://neurips.cc/public/EthicsGuidelines}?
    \item[] Answer: \answerYes{} 
    \item[] Justification: There are no ethical violations, to the authors' knowledge.

\item {\bf Broader impacts}
    \item[] Question: Does the paper discuss both potential positive societal impacts and negative societal impacts of the work performed?
    \item[] Answer: \answerYes{} 
    \item[] Justification: An ``Impact Statement'' discussing potential benefits and risks is included before the references.
    
\item {\bf Safeguards}
    \item[] Question: Does the paper describe safeguards that have been put in place for responsible release of data or models that have a high risk for misuse (e.g., pre-trained language models, image generators, or scraped datasets)?
    \item[] Answer: \answerNA{} 
    \item[] Justification: We do not provide any models or datasets in this paper. Our contributions are theoretical and methodological.

\item {\bf Licenses for existing assets}
    \item[] Question: Are the creators or original owners of assets (e.g., code, data, models), used in the paper, properly credited and are the license and terms of use explicitly mentioned and properly respected?
    \item[] Answer: \answerYes{} 
    \item[] Justification: The paper cites the UCI Apartment dataset and identifies it as licensed under CC BY 4.0. Meta Platforms, Inc. approved the publication of aggregate MAU evaluation results, but the underlying proprietary data are not redistributed.

\item {\bf New assets}
    \item[] Question: Are new assets introduced in the paper well documented and is the documentation provided alongside the assets?
    \item[] Answer: \answerYes{} 
    \item[] Justification: The released GitHub repository includes documentation for the code.

\item {\bf Crowdsourcing and research with human subjects}
    \item[] Question: For crowdsourcing experiments and research with human subjects, does the paper include the full text of instructions given to participants and screenshots, if applicable, as well as details about compensation (if any)? 
    \item[] Answer: \answerNA{} 
    \item[] Justification: Experiments of this nature are not included in this paper.

\item {\bf Institutional review board (IRB) approvals or equivalent for research with human subjects}
    \item[] Question: Does the paper describe potential risks incurred by study participants, whether such risks were disclosed to the subjects, and whether Institutional Review Board (IRB) approvals (or an equivalent approval/review based on the requirements of your country or institution) were obtained?
    \item[] Answer: \answerNA{} 
    \item[] Justification: The experiments use publicly available tabular data and proprietary country-level aggregate statistics. They do not involve interaction with human subjects or collection of identifiable participant data. Meta Platforms, Inc. approved publication of the aggregate evaluation results.

\item {\bf Declaration of LLM usage}
    \item[] Question: Does the paper describe the usage of LLMs if it is an important, original, or non-standard component of the core methods in this research? Note that if the LLM is used only for writing, editing, or formatting purposes and does \emph{not} impact the core methodology, scientific rigor, or originality of the research, declaration is not required.
    \item[] Answer: \answerNA{} 
    \item[] Justification: No substantive part of this research involved the use of large language models.

\end{enumerate}

\end{document}